\documentclass{article}

\usepackage[final]{neurips_2019}

\usepackage[utf8]{inputenc} 
\usepackage[T1]{fontenc}    
\usepackage{hyperref}       
\usepackage{url}            
\usepackage{booktabs}       
\usepackage{amsfonts}       
\usepackage{nicefrac}       
\usepackage{microtype}      
\usepackage{appendix}
\usepackage{hyperref}
\usepackage{amsthm}
\usepackage{amsmath}
\usepackage{amssymb}
\usepackage{amsfonts}
\usepackage{enumerate}
\usepackage{algorithm}
\usepackage{algorithmic} 
\usepackage{geometry}
\newtheorem{definition}{Definition}

\newtheorem{lemma}{Lemma}

\newtheorem{thm}{Theorem}
\newtheorem{assumption}{Assumption}
\newtheorem{corollary}{Corollary}

\title{Regret Minimization for Reinforcement Learning by Evaluating the Optimal Bias Function}

\author{%
  Zihan Zhang \\
  Tsinghua University\\
  \texttt{zihan-zh17@mails.tsinghua.edu.cn} \\
   \And
   Xiangyang Ji \\
  Tsinghua University\\
   \texttt{xyji@tsinghua.edu.cn} \\
}

\begin{document}

\maketitle

\begin{abstract}
We present an algorithm based on the \emph{Optimism in the Face of Uncertainty} (OFU) principle which is able to learn Reinforcement Learning (RL) modeled by Markov decision process (MDP) with finite state-action space efficiently. 
By evaluating the state-pair difference of the optimal bias function $h^{*}$, the proposed algorithm achieves a regret bound of  $\tilde{O}(\sqrt{SAHT})$\footnote{The symbol $\tilde{O}$ means $O$ with log factors ignored. } for MDP with $S$ states and $A$ actions, in the case that an upper bound $H$ on the span of $h^{*}$, i.e., $sp(h^{*})$ is known. 
This result outperforms the best previous regret bounds $\tilde{O}(S\sqrt{AHT})  $\citep{fruit2019improved} by a factor of $\sqrt{S}$. 
Furthermore, this regret bound matches the lower bound of $\Omega(\sqrt{SAHT}) $\citep{jaksch2010near} up to a logarithmic factor. As a consequence,  we show that there is a near optimal regret bound of $\tilde{O}(\sqrt{SADT})$ for MDPs with a finite diameter $D$ compared to the lower bound of $\Omega(\sqrt{SADT}) $\citep{jaksch2010near}.
\end{abstract}
\section{Introduction}
In this work we consider the Reinforcement Learning (RL) problem \citep{Burnetas1997Optimal,sutton2018reinforcement} of an agent interacting with an environment. 
The problem is generally modelled as a discrete Markov Decision Process (MDP)\citep{Puterman1994Markov}. The RL agent needs to learn the underlying dynamics of the environment in order to make sequential decisions. At step $t$, the agent observes current state $s_{t}$ and chooses an action $a_{t}$ based on the policy learned from the past. Then the agent receives a reward $r_{t}$ from the environment, and the environment transits to state $s_{t+1}$ according to the states transition model.
Particularly, both $r_{t}$ and $s_{t+1}$ are independent of previous trajectories, and are only conditioned on $s_{t}$ and $a_{t}$. 
In the online framework of reinforcement learning, we aim to maximize  cumulative reward. Therefore, 
there is a trade-off between \emph{exploration} and \emph{exploitation}, i.e., taking actions we have not learned accurately enough and taking actions which seem to be optimal currently.

The solutions to exploration-exploitation dilemma can mainly be divided into two groups. 
In the first group, the approaches utilize the \emph{Optimism in the Face of Uncertainty} (OFU) principle \citep{auer2002finite}. Under OFU principle, the agent maintains a confident set of MDPs and the underlying MDP is contained in this set with high probability. The agent executes the optimal policy of the best MDP in the confidence set \citep{bartlett2009regal,jaksch2010near,maillard2011finite,fruit2018near}. 
In the second group, the approaches utilize posterior sampling \citep{thompson1933likelihood}. The agent maintains a posterior distribution over reward functions and transition models.  It samples an MDP and executes corresponding optimal policy in each epoch. Because of simplicity and scalability, as well as provably optimal regret bound,  posterior sampling has been getting popular in related research field \citep{Osband2013,osband2016posterior,agrawal2017optimistic,Abbasi2015Bayesian}.

\subsection{Related Work}
In the research field of regret minimization for reinforcement learning, \citet{jaksch2010near} showed a regret bound of $\tilde{O}(DS\sqrt{AT})$ for MDPs  with a finite diameter $D$, and proved that it is impossible to reach a regret bound smaller than $\Omega(\sqrt{SADT})$. \citet{agrawal2017optimistic} established a better regret bound of $\tilde{O}(D\sqrt{SAT})$ by posterior sampling method. 
 \citet{bartlett2009regal} achieved a regret bound of $\tilde{O}(HS\sqrt{AT})$ where $H$ is  an input as an upper bound of $sp(h^*)$  . \citet{fruit2018efficient} designed a practical algorithm for the constrained optimization problem in R\scriptsize{EGAL}\normalsize.C \citep{bartlett2009regal}, and obtained a regret bound of $\tilde{O}(H\sqrt{\Gamma SAT})$ where $\Gamma\leq S$ is the number of possible next states. 
On the other hand, 
\citet{ouyang2017learning} and \citet{theocharous2017posterior} designed posterior sampling algorithms with Bayesian regret bound of $\tilde{O}(HS\sqrt{AT})$, with the assumption that  elements of  support of the prior distribution have a consistent upper bound $H$ for their optimal bias spans.\cite{talebi2018variance} showed a problem-dependent regret bound of $\tilde{O}(\sqrt{\sum_{s,a}V(P_{s,a},h^*)ST})$. Recently, \cite{fruit2019improved} presented  improved analysis of UCRL2B algorithm and obtained a regret bound of $\tilde{O}(S\sqrt{DAT})$.

There are also considerable work devoted to studying finite-horizon MDP.
\citet{osband2016posterior} presented PRSL to establish a Bayesian regret bound of $\tilde{O}(H\sqrt{SAT})$ using posterior sampling method. And later \citet{azar2017minimax} reached a better regret bound of $\tilde{O}(\sqrt{SAHT})$. Recently, \citet{kakade2018variance} and \citet{zanette2019tighter} achieved the same regret bound of $\tilde{O}(\sqrt{SAHT})$ by learning a precise value function to predict the best future reward of current state.


 We notice a mistake about concentration of average of independent multinoulli trials in the proof of \citep{agrawal2017optimistic} (see Appendix.A for further details). This mistake suggests that they may not reduce a factor of $\sqrt{S}$ in their regret bounds.

\subsection{Main Contribution}
In this paper, we design an OFU based algorithm, and achieve a regret bound of $\tilde{O}(\sqrt{SAHT})$ given an upper bound $H$ on $sp(h^{*})$. 
As a corollary, we establish a regret bound of $\tilde{O}(\sqrt{SADT})$ for the MDPs with finite diameter $D$. 
Meanwhile the corresponding lower bounds for the above two upper bounds are $\Omega(\sqrt{SAHT})$ and $\Omega(\sqrt{SADT})$ respectively. 
In a nutshell, our algorithm improves the regret bound by a factor of $\sqrt{S}$ compared to the best previous known results.

\textbf{Our Approach:} we consider regret minimization for RL by evaluating  state-pair difference of the optimal bias function. Firstly, we observe that we can achieve a near-optimal regret bound with  guide of the optimal bias function. Considering the fact that it is hard to estimate the optimal bias function directly \citep{ortner2008online},  we  design a confidence set $\mathcal{H}_{k}$ of the optimal bias function. Based on $\mathcal{H}_{k}$ we obtain a tighter confidence set of  MDPs and a better regret bound. It is notable that the order of samples in the trajectory is crucial when computing $\mathcal{H}_{k}$ in our algorithm, while it is ignored in previous methods. In this way, we utilize more information about the trajectory when computing the confidence set, which enables us to achieve a better regret bound.



\section{Preliminaries}
We consider the MDP learning problem where the MDP $M=\langle\mathcal{S},\mathcal{A},r,P,s_{1}\rangle$. 
$\mathcal{S}=\{1,2,...,S\}$ is the state space, $\mathcal{A}=\{1,2,...,A\}$ is the action space, $P:\mathcal{S}\times\mathcal{A}\to \Delta^{\mathcal{S}}$\footnote{In this paper, we use $\Delta^{X}$ to denote all distributions on $X$. Particularly, we use $\Delta^{m}$ to denote the $m$-\emph{simplex}.} is the transition model, $r:\mathcal{S}\times \mathcal{A}\to \Delta^{[0,1]}$ is the reward function, and $s_{1}$ is the initial state. 
The agent executes action $a$ at state $s$ and receives a reward $r(s,a)$, and then the system transits to the next state $s'$ according to $\mathbb{P}(\cdot|s,a)=P_{s,a}$. 
In this paper, we assume that $\mathbb{E}[r(s,a)]$ is known for each $(s,a)$ pair, and denote $\mathbb{E}[r(s,a)]$ as $r_{s,a}$. 
It is not difficult to extend the proof to the original case.  

In the following sections, we mainly focus on weak-communicating (see definition \citep{bartlett2009regal}) MDPs.

\begin{assumption}\label{assumption1}
The underlying MDP is weak-communicating  .
\end{assumption}
We first summarize several useful known results for MDPs and RL. 
\begin{definition}[Policy]
A policy $\pi: \mathcal{S}\to \Delta^{\mathcal{A}}$  is a mapping from the state space to all distributions on the action space. In the case the support of $\pi(s)$ is a single action, we also denote this action as $\pi(s)$.
\end{definition}
Given a policy $\pi$, transition model $P$ and reward function $r$, we use $P_{\pi}$ to denote the transition probability matrix and $r_{\pi}$ to denote the reward vector under $\pi$. Specifically, when $\pi$ is a deterministic policy, $P_{\pi}=[P_{1,\pi(1)},...,P_{s,\pi(s)}]$ and  $r_{\pi}=[r_{1,\pi(1)},...,r_{S,\pi(S)}]^{T}$.

\begin{definition}[Average reward]
Given a policy $\pi$, when starting from $s_{1}=s$, the average reward is defined as:
$$\rho_{\pi}(s)=\mathop{lim} \limits_{T\to \infty}\frac{1}{T}\mathbb{E}_{a_{t}\sim \pi(s_{t}) ,1\leq t\leq T}[\sum_{t=1}^{T}r_{s_{t},a_{t}}|s_{1}=s].$$
\end{definition}
The optimal average reward and the optimal policy are defined as
$\rho^{*}(s)=\max_{\pi}\rho_{\pi}(s)$ and $\Pi^{*}(s)=\mathop{\arg\max}_{\pi}\rho_{\pi}(s)$
respectively. It is well known that, under Assumption \ref{assumption1}, $\rho^{*}(s)$ is state independent, so that we write it as $\rho^{*}$ in the rest of the paper for simplicity.

\begin{definition}[Diameter] Diameter of an MDP $M$ is defined as:
$$D(M)=\max\limits_{s,s'\in \mathcal{S},s\neq s'}\mathop{\min}\limits_{\pi:\mathcal{S}\to \Delta_{\mathcal{A}}}T^{\pi}_{s\to s'},$$
where $T^{\pi}_{s\to s'}$ denotes the expected number of steps to reach $s'$ from $s$ under policy $\pi$.
\end{definition}


 
Under Assumption \ref{assumption1}, it is known the optimal bias function $h^{*}$ satisfies that
\begin{equation}\label{eq1}
    h^{*}+\rho^{*}\textbf{1}=\max_{a\in \mathcal{A}}(r_{s,a}+P_{s,a}^{T}h^{*})
\end{equation}
 where $\textbf{1}=[1,1,...,1]^{T}$. It is obvious that if $h$ satisfies  (\ref{eq1}), then so is $h^{*}+\lambda \textbf{1}$ for any $\lambda\in \mathbb{R}$. Assuming $h$ is a solution to  (\ref{eq1}), we set\footnote{In this paper, we use $[v_{1},v_{2},...,v_{S}]^{T}$ to indicate a vector $v\in \mathbb{R}^{\mathcal{S}}$} $\lambda=-\min_{s}h_{s}$ and $h^{*}=h+\lambda \textbf{1}$, then the optimal bias function $h^{*}$ is uniquely defined.
Besides, the span operator $sp:\mathbb{R}^{S}\to \mathbb{R}$ is defined as $sp(v)=\max\limits_{s,s'\in [S]}|v_{s}-v_{s'}|$.



\textbf{The reinforcement learning problem.} In reinforcement learning, the agent starts at $s_{1}=s_{start}$, and proceeds to make decisions in rounds $t=1,2,...,T$. The $\mathcal{S}$, $\mathcal{A}$ and $\{r_{s,a}\}_{s\in \mathcal{S},a\in \mathcal{A}}$ are known to the agent, while the transition model $P$ is unknown to agent. 
Therefore, the final performance is measured by the cumulative regret defined as
$$\mathcal{R}(T,s_{start}):=T\rho^{*}-\sum_{t=1}^{T}r_{s_{t},a_{t}}.$$
The upper bound for $\mathcal{R}(T,s_{start})$ we provide is always consistent with that of $s_{start}$. In the following sections, we use $\mathcal{R}(T,s_{start})$ to denote $\mathcal{R}(T)$ for simplicity. 

\section{Algorithm Description}\label{algds}
\subsection{Framework of UCRL2}
We first revisit the classical framework of UCRL2 \citep{jaksch2010near} briefly. As described in Algorithm \ref{alg1} (\emph{EBF}), there are mainly three components in the UCRL2 framework: \emph{doubling episodes}, \emph{building the confidence set} and \emph{solving the optimization problem}. 

\textbf{Doubling episodes}: The algorithm proceeds through episodes $k=1,2,...$. In the $k$-th episode, the agent makes decisions according to $\pi_{k}$. The episode ends whenever $\exists (s,a)$, such that the visit count of $(s,a)$ in the $k$-th episode is larger than or equal to the visit count of $(s,a)$ before the $k$-th episode. Let $K$ be the number of episodes. Therefore, we can get that $K\leq SA(\log_{2}(\frac{T}{SA})+1)\leq  3SA\log(T)$ when $SA\geq 2$ \citep{jaksch2010near}.

\textbf{Building the confidence set:} At the beginning of an episode, the algorithm computes a collection of plausible MDPs, i.e., the confidence set $\mathcal{M}_{k}$ based on previous trajectory. $\mathcal{M}_{k}$ should be designed properly such that the underlying MDP $M$ is contained by $\mathcal{M}_{k}$ with high probability, and the elements in $\mathcal{M}_{k}$ are closed to $M$. 
In our algorithm, the confidence set is not a collection of MDPs. Instead, we design a 4-tuple $(\pi,P'(\pi),h'(\pi),\rho(\pi))$ to describe a plausible MDP and its optimal policy.

\textbf{Solving the optimization problem:} Given a confidence set $\mathcal{M}$, the algorithm selects an element from $\mathcal{M}$ according to some criteria. Generally, to keep the optimality of the chosen MDP, the algorithm needs to maximize the average reward with respect to certain constraints. Then the corresponding optimal policy will be executed in current episode.
 
\subsection{Tighter Confidence Set by Evaluating the Optimal Bias Function}
R\scriptsize{EGAL}\normalsize.C \citep{bartlett2009regal} utilizes $H$ to compute $\mathcal{M}_{k}$, thus avoiding the issues brought by the diameter $D$. Similar to R\scriptsize{EGAL}\normalsize.C, we assume that $H$, an upper bound of $sp(h^{*})$ is known.  We design a novel method to compute the confidence set, which is able to utilize the knowledge of the history trajectory more efficiently. We first compute a well-designed confidence set $\mathcal{H}_{k}$ of the optimal bias function, and   obtain a tighter confidence set $\mathcal{M}_{k}$ based on $\mathcal{H}_{k}$.

On the basis of above discussion, we summarize high-level intuitions as below:

\textbf{Exploration guided by the optimal bias function:} Once the true optimal bias function $h^{*}$ is given, we could get a better regret bound. In this case we regard the regret minimization problem as $S$ independent multi-armed bandit problems. UCB algorithm with Bernstein bound \citep{lattimore2012pac} provides a near optimal regret bound. However, we can not get $h^*$ exactly. Instead,  a tight confidence set of $h^*$ also helps to guide exploration.

\textbf{Confidence set of the optimal bias function:} We first study what could be learned about $h^*$ if we always choose optimal actions. For two different states $s,s'$, suppose we start from $s$ at $t_{1}$, and reach $s'$ the first time at $t_{2}$ ($t_{2}$ is a stopping time), then we have $\mathbb{E}[\sum_{t=t_{1}}^{t_{2}-1}(r_{t}-\rho^*)]$\footnote{To explain the high-level idea, we assume this expectaion is well-defined.}$=\delta^*_{s,s'}:=h^*_{s}-h^*_{s'}$ by the definition of optimal bias function. As a result, $\sum_{t=t_{1}}^{t_{2}-1}(r_{t}-\rho^*)$ could be regarded as an unbiased estimator for $\delta^*_{s,s'}$. Based on concentration inequalities for martingales, we have the following formal definitions and lemma.

\begin{definition}\label{def4}
Given a trajectory $\mathcal{L}=\{(s_{t},a_{t},s_{t+1},r_{t})\}_{1\leq t\leq N}$, for $s, s' \in \mathcal{S}$ and $s \neq s'$, let $ts_{1}(\mathcal{L}):=\min\{\min\{t|s_{t}=s\},N+2\}$. We define $\{ts_{k}(\mathcal{L})\}_{k\geq 2}$ and $\{te_{k}(\mathcal{L})\}_{k\geq 1}$ recursively by following rules,
$$te_{k}(\mathcal{L}):=\min \big\{\min\{t|s_{t}=s',t>ts_{k}(\mathcal{L})\},N+2 \big\},$$
$$ts_{k}(\mathcal{L}):=\min \big\{\min\{t|s_{t}=s,t>te_{k-1}(\mathcal{L})\},N+2 \big\}.$$
The count of arrivals $c(s,s',\mathcal{L})$ from $s$ to $s'$ is defined as
$$c(s,s',\mathcal{L}):=\max\{k|te_{k}(\mathcal{L})\leq N+1\}.$$
Here we define $\min \varnothing=+\infty$ and $\max \varnothing=0$ respectively.

\end{definition}
\begin{lemma}[Main Lemma]\label{keylemma} 
We say an MDP is flat if all its actions are optimal.
Suppose $M$ is a flat MDP (without the constraint $r_{s,a}\in [0,1]$).  We run $N$ steps following an algorithm $\mathcal{G}$ under  $M$. Let $\mathcal{L}=\{(s_{t},a_{t},s_{t+1},r_{t} )\}_{1\leq t\leq N}$ be the final trajectory. For any two states $s, s' \in \mathcal{S}$ and $s \neq s'$, let $c(s,s',\mathcal{L})$, $\{te_{k}(\mathcal{L})\}_{k\geq 1}$ and $\{ts_{k}(\mathcal{L})\}_{k\geq 1}$ be defined as in Definition \ref{def4}. Then we have, for any algorithm $\mathcal{G}$, with probability at least $1-N\delta$, for any $1\leq c\leq c(s,s',\mathcal{L})$ it holds that 
\begin{equation}\label{eq122}
    \begin{aligned}
    \Big|\sum_{k=1}^{c}\Big(h^{*}_{s'}-h^{*}_{s}+\sum_{ts_{k}(\mathcal{L})\leq t \leq te_{k}(\mathcal{L})-1}(r_{t}-\rho^{*})\Big )\Big|\leq (\sqrt{2N\gamma }+1)sp(h^{*}).
    \end{aligned}
\end{equation}
where $\gamma = \log(\frac{2}{\delta})$\footnote{In this paper $\gamma$ always denotes $\log(\frac{2}{\delta})$. }
\end{lemma}
To use Lemma \ref{keylemma} to compute $\mathcal{H}_{k}$, we have to overcome two problems: (\romannumeral1) $M$ may not be \emph{flat}; (\romannumeral2) we do not have the value of $\rho^*$.  Under the assumption the total regret is $\tilde{O}(HS\sqrt{AT})$, we can solve the problems subtly. 

 Let $reg_{s,a}=h^{*}_{s}+\rho^{*}-P_{s,a}^{T}h^{*}-r_{s,a}$, which is also called optimal gap \citep{Burnetas1997Optimal} and could be regarded as the single step regret of $(s,a)$. Let  $r'_{s,a}=h^{*}_{s}+\rho^*-P_{s,a}^{T}h^{*}=r_{s,a}+reg_{s,a}$ and $M' = \langle \mathcal{S},\mathcal{A},r',P,s_{1}  \rangle$. It is easy to prove that  $M'$ is \emph{flat} and has the same optimal bias function and optimal average reward as $M$. 
    We attain by Lemma \ref{keylemma} that with high probability, it holds that
\begin{equation}\label{neweq1}
   \Big|\sum_{k=1}^{c(s,s',\mathcal{L})}\Big(h^{*}_{s'}-h^{*}_{s}+\sum_{ts_{k}(\mathcal{L})\leq t \leq te_{k}(\mathcal{L})-1}(r_{s_{t},a_{t}}-\rho^{*})\Big )\Big|\leq \sum_{t=1}^{N}reg_{s_{t},a_{t}}+ (\sqrt{2N\gamma }+1)sp(h^{*}).
\end{equation}
      Let $h'\in [0,H]^{S}$ be a vector such that  (\ref{neweq1}) still holds with $h^{*}$ replaced by $h'$, then we can derive that
\begin{equation*}
N_{s,a,s'}|(h^*_{s'}-h^*_{s})- (h'_{s'}-h'_{s})|\leq  2\sum_{t=1}^{N}reg_{s_{t},a_{t}}+2(\sqrt{2N\gamma}+1)H
\end{equation*}
where $ N_{s,a,s'}:=\sum_{t=1}^{N}\mathbb{I}[s_{t}=s,a_{t}=a,s_{t+1}=s']\leq c(s,s',\mathcal{L})$.
Because it is not hard to bound $\sum_{t=1}^{N}reg_{s_{t},a_{t}}\approx \mathcal{R}(N)$ up to $\tilde{O}(HS\sqrt{AN})$ by R\scriptsize{EGAL}\normalsize.C \citep{bartlett2009regal}, we obtain that
with high probability it holds
\begin{equation}\label{hbound}
\hat{N}_{s,a,s'}|(h^*_{s'}-h^*_{s})- (h'_{s'}-h'_{s})|=\tilde{O}(HS\sqrt{AN}).
\end{equation}
As for the problem we have no knowledge about $\rho^{*}$, we can replace $\rho^*$ with the empirical average reward $\hat{\rho}$.  Our claim about (\ref{hbound}) still holds as long as $N(\rho^{*}-\hat{\rho})=\tilde{O}(HS\sqrt{AN})$, which is equivalent to  $\mathcal{R}(N)=\tilde{O}(HS\sqrt{AN})$.  

Although it seems that  (\ref{hbound}) is not tight enough,  it helps to bound the error term due to the difference between $h_{k}$ and $h^{*}$ up to $o(\sqrt{T})$ by setting $N=T$. (refer to Appendix.C.5.)

Based on the discussion above, we define $\mathcal{H}_{k}$ as:
$$\mathcal{H}_{k}:=\{h\in [0,H]^S| |L_{1}(h,s,s',\mathcal{L}_{t_{k}-1})|\leq 48S\sqrt{AT}sp(h)+(\sqrt{2\gamma T}+1)sp(h),\forall s,s',s\neq s'  \}$$
where
\begin{equation*}
\begin{aligned}
&L_{1}(h,s,s',\mathcal{L})=\sum_{k=1}^{c(s,s',\mathcal{L})}\Big((h_{s'}-h_{s})+ \sum_{ts_{k}(\mathcal{L})\leq i\leq te_{k}(\mathcal{L})-1}(r_{i}-\hat{\rho})\Big).
\end{aligned}
\end{equation*}
Together with  constraints on the transition model (\ref{a2c2})-(\ref{a2c4}) and constraint on optimality (\ref{a2c5}), we propose Algorithm \ref{alg2} to build the confidence set, where
\begin{equation*}
\begin{aligned}
&V(x,h)=\sum_{s}x_{s}h^{2}_{s}-(x^{T}h)^{2}.
\end{aligned}
\end{equation*}

\begin{algorithm}[tb]
   \caption{EBF: Estimate the Bias Function}
   \label{alg:example1}
\hspace*{0.02in} {\bf Input:}  $H$, $\delta$, $T$.\\
   \hspace*{0.02in} {\bf Initialize:}  $t\leftarrow1$,$t_{k}\leftarrow 0$.
\begin{algorithmic}[1]\label{alg1}
   \FOR{episodes $k=1,2,...$ }
   \STATE{$t_{k}\leftarrow$current time;}
   \STATE{$\mathcal{L}_{t_{k}-1}\leftarrow \{(s_{i},a_{i},s_{i+1},r_{i})\}_{1\leq i\leq t_{k}-1}$;}
   \STATE{$\mathcal{M}_{k}\leftarrow
   $\emph{BuildCS}$(H,\log(\frac{2}{\delta}), \mathcal{L}_{t_{k}-1})$;}
   \STATE{Choose $(\pi,P'(\pi),h'(\pi),\rho(\pi))\in \mathcal{M}_{k}$ to maximize $\rho(\pi)$  over $\mathcal{M}_{k}$;}
   \STATE{$\pi_{k}\leftarrow$ $\pi$;}
   \STATE{Follow $\pi_{k}$ until the visit count of some $(s,a)$ pair doubles.}
   \ENDFOR
\end{algorithmic}
\end{algorithm}
\begin{algorithm}[tb]
   \caption{BuildCS($H$,$\gamma$, $\mathcal{L}$)}
   \label{alg:example2}
   \hspace*{0.02in} {\bf Input:} 
   $H$, $\gamma$, $\mathcal{L}=\{(s_{i},a_{i},s_{i+1},r_{i})\}_{1\leq i \leq N}$
\begin{algorithmic}[1]\label{alg2}
  
   \STATE{$\mathcal{H}\leftarrow \{h\in [0,H]^{S}|\,\,|L_{1}(h,s,s',\mathcal{L})|\leq 48S\sqrt{AT}sp(h)+(\sqrt{2\gamma T}+1)sp(h)
   ,\forall s,s',s\neq s' \}$;}
   \STATE{$N_{s,a}\leftarrow \max\{\sum_{t=1}^{N}\mathbb{I}[s_{t}=s,a_{t}=a],1\}$, $\forall (s,a)$;}
   \STATE{$\hat{P}_{s,a,s'}\leftarrow \frac{  \sum_{t=1}^{N}\mathbb{I}[s_{t}=s,a_{t}=a,s_{t+1}=s']}{N_{s,a}}$, $\forall (s,a,s')$;}
   \STATE{$\mathcal{O}\leftarrow \{\pi|\pi \mbox{ is a deterministic policy, and }  \exists P'(\pi)\in \mathbb{R}^{S\times A\times S}, h'(\pi)\in \mathcal{H}\mbox{ and } \rho(\pi)\in \mathbb{R}, \mbox{such that} $
   \begin{equation}\label{a2c2}
   \begin{aligned}
        &|P'_{s,a,s'}(\pi)-\hat{P}_{s,a,s'}|\leq 2\sqrt{\hat{P}_{s,a,s'}\gamma/N_{s,a}}+3\gamma/N_{s,a}+4\gamma^{\frac{3}{4}}/N_{s,a}^{\frac{3}{4}},
   \end{aligned}
   \end{equation}
   \begin{equation}\label{a2c3}
       |P'_{s,a}(\pi)-\hat{P}_{s,a}|_{1}\leq \sqrt{14S\gamma/N_{s,a}}
   \end{equation}
   \begin{equation}\label{a2c4}
       |(P'_{s,a}(\pi)-\hat{P}_{s,a})^{T}h'(\pi)|\leq 2\sqrt{V(\hat{P}_{s,a},h'(\pi))\gamma/N_{s,a}}+12H\gamma/N_{s,a}+10H\gamma^{3/4}/N_{k,s,a}^{3/4},
   \end{equation}
   \begin{equation}\label{a2c5}
       P'_{s,\pi(s)}(\pi)^{T}h'(\pi)+r_{s,\pi(s)}=\mathop{max}\limits_{a\in \mathcal{A}}P'_{s,a}(\pi)^{T}h'(\pi)+r_{s,a}=h'(\pi)+\rho(\pi)\textbf{1} 
   \end{equation}
   \mbox{holds for any} $s,a,s' \}$;}
   \STATE{{\bfseries Return:}\{$(\pi,P'(\pi),h'(\pi),\rho(\pi))|\pi \in \mathcal{O}$\}.}
\end{algorithmic}
\end{algorithm}

\section{Main Results}
In this section, we summarize the results obtained by using Algorithm \ref{alg1} on weak-communicating MDPs. In the case there is an available upper bound $H$ for $sp(h^{*})$, we have following theorem.
\begin{thm}[Regret bound ($H$ known)]
With probability $1-\delta$, for any weak-communicating MDP $M$ and any initial state $s_{start}\in \mathcal{S}$, 
when $T\geq p_{1}(S,A,H,\log(\frac{1}{\delta}))$ and $S,A,H\geq 20$ where $p_{1}$ is a polynomial function, 
the regret of EBF algorithm is bounded by
$$\mathcal{R}(T)\leq 490\sqrt{SAHT\log(\frac{40S^{2}A^{2}T\log(T)}{\delta})},$$
whenever an upper bound of the span of optimal bias function $H$ is known. By setting $\delta=\frac{1}{T}$, we get that $\mathbb{E}[\mathcal{R}(T)]=\tilde{O}(\sqrt{SAHT})$
\end{thm}
Theorem 1 generalizes the $\tilde{O}(\sqrt{SAHT})$ regret bound from the finite-horizon setting \citep{azar2017minimax} to general weak-communicating MDPs, and improves the best previous known regret bound $\tilde{O}(H\sqrt{SAT})$\citep{fruit2019improved} by an $\sqrt{S}$ factor. 
More importantly, this upper bound matches the $\Omega(\sqrt{SAHT})$ lower bound up to a logarithmic factor.

 
Based on Theorem 1, in the case the diameter $D$ is finite but unknown, we can reach a regret bound of $\tilde{O}(\sqrt{SADT})$.
\begin{corollary}\label{coro1}
For weak-communicating MDP $M$ with a finite unknown diameter $D$ and any initial state $s_{start}\in \mathcal{S}$,
with probability $1-\delta$, when $T\geq p_{2}(S,A,D,\log(\frac{1}{\delta}))$ and $S,A,D\geq 20$ where $p_{2}$ is a polynomial function,
the regret can be bounded by
$$\mathcal{R}(T)\leq 491\sqrt{SADT(\log(\frac{S^{3}A^{2}T\log(T)}{\delta})}.$$
By setting $\delta=\frac{1}{T}$, we get that $\mathbb{E}[\mathcal{R}(T)]=\tilde{O}(\sqrt{SADT})$.
\end{corollary}
We postpone the proof of Corollary \ref{coro1} to Appendix.D.

Although \emph{EBF} is proved to be near optimal, it is hard to implement the algorithm efficiently.   The optimization problem in line 5 Algorithm \ref{alg1} is well-posed because of the optimality equation (\ref{a2c5}). However, the constraint (\ref{a2c4}) is non-convex in $h'(\pi)$, which makes the optimization problem hard to solve. Recently, \citet{fruit2018efficient} proposed a practical algorithm SCAL, which solves the optimization problem in R\scriptsize{EGAL}\normalsize.C efficiently. We try to expand the \emph{span truncation} operator $T_{c}$ to our framework, but fail to make substantial progress. We have to leave this to future work.

\section{Analysis of EBF (Proof Sketch of Theorem 1) }
Our proof mainly contains two parts. In the first part, we bound the probabilites of the bad events.  In the second part, we manage to bound the regret when the good event occurs.

\subsection{Probability of Bad Events}
We first present the explicit definition of the bad events. Let $N^{(t)}_{s,a}=\sum_{i=1}^{t}\mathbb{I}[s_{i}=s,a_{i}=a]$. 
We denote $N_{k,s,a}=N_{s,a}^{(t_{k}-1)}$ as the visit count of $(s,a)$ before the $k$-th episode, and $v_{k,s,a}$ as the visit count of $(s,a)$ in the $k$-th episode respectively. We also denote $\hat{P}^{(k)}$ as the empirical transition model before the $k$-th episode.

\begin{definition}[Bad event] For the $k$-th episode, define
	\begin{equation*}
	\begin{aligned}
	&B_{1,k}:=\bigg\{\exists (s,a), s.t. |(P_{s,a}-\hat{P}^{(k)}_{s,a})^{T}h^{*}|> 2\sqrt{\frac{V(P_{s,a},h^{*})\gamma)}{\max\{N_{k,s,a},1\}}}+2\frac{sp(h^{*}\gamma)}{\max\{N_{k,s,a},1\}}  \bigg\},\\
	&B_{2,k}=\bigg\{\exists (s,a,s'), s.t. |\hat{P}^{(k)}_{s,a,s'}-P_{s,a,s'}|> 2\sqrt{\frac{\hat{P}^{(k)}_{s,a,s'}\gamma}{\max\{N_{k,s,a},1\}}}+\frac{3\gamma}{\max\{N_{k,s,a},1\}}+\frac{4\gamma^{\frac{3}{4}}}{\max\{N_{k,s,a},1\}^{\frac{3}{4}}}  \bigg\},\\
	 &B_{3,k}=\Big\{|\sum_{1\leq t<t_{k}}(\rho^{*}-r_{s_{t},a_{t}})|> 26HS\sqrt{AT\gamma},\sum_{k'<k}\sum_{s,a}v_{k',s,a}reg_{s,a}> 22HS\sqrt{AT\gamma}  \Big\}\\
	& B_{4,k}=\big\{\{(\pi^{*},P^{*},h^{*},\rho^{*})|\pi^{*} \mbox{is a deterministic optimal policy}\}\cap \mathcal{M}_{k}=\varnothing  \big\}.
	\end{aligned}
	\end{equation*}
	The bad event in the $k$-th episode therefore is defined as $B_{k} =  B_{1,k}\cup B_{2,k}\cup B_{3,k}\cup B_{4,k}$, and the total bad event $B$ is defined as $B:=\cup_{1\leq k\leq K+1}B_{k}$. At the same time, we have the definition of the good event as $G=B^{C}$.
\end{definition}
\begin{lemma}[Bound of $\mathbb{P}(B)$] Suppose we run Algorithm \ref{alg1} for $T$ steps, then $\mathbb{P}(B)\leq (6AT+12S^{2}A)SA\log(T)\delta$ when $T\geq A\log(T)$ and $SA\geq 4$.
\end{lemma}
\subsection{Regret when the Good Event Occurs}
In this section we assume that the good event $G$ occurs. We use $\mathcal{R}_{k}$ to denote the regret in the $k$-th episode. We use $P'_{k}$, $P_{k}$, $\hat{P}_{k}$, $r_{k}$, $\rho_{k}$ and $h_{k}$   to denote $P'_{\pi_{k}}(\pi_{k})$, $P_{\pi_{k}}$, $\hat{P}^{(k)}_{\pi_{k}}$, $r_{\pi_{k}}$, $\rho(\pi_{k})$ and $h'(\pi_{k})$ respectively. We define $v_{k}$ as the vector such that $v_{k,s}=v_{k,s,\pi_{k}(s)},\forall s$, and introduce $\delta_{k,s,s'}=h_{k,s}-h_{k,s'},\forall s,s'$. 

Noting that for $\alpha>0$, $\sum_{k}\sum_{s,a}v_{k,s,a}\frac{1}{\max\{N_{k,a,s},1\}^{\frac{1}{2}+\alpha}}$ could be roughly bounded by  $O(T^{\frac{1}{2}-\alpha})$, which could be ignored when $T$ is sufficiently large. Therefore, we can omit such terms without changing the regret bound.

According to $B_{4,k}^{C}$ and the optimality of $\rho_{k}$ we have
\begin{equation}\label{eq5.2.1}
\begin{aligned}
\mathcal{R}_{k}&=v_{k}^{T}(\rho^{*}\textbf{1}-r_{k})\leq v_{k}^{T}(\rho_{k}\textbf{1}-r_{k})=v_{k}^{T}(P'_{k}-I)^{T}h_{k} \\
&=\underbrace{v_{k}^{T}(P_{k}-I)^{T}h_{k}}_{\textcircled{1}_{k}}+\underbrace{v_{k}^{T}(\hat{P}_{k}-P_{k})^{T}h^{*}}_{\textcircled{2}_{k}}+ \underbrace{v_{k}^{T}(P'_{k}-\hat{P}_{k})^{T}h_{k}}_{\textcircled{3}_{k}}+\underbrace{v_{k}^{T}(\hat{P}_{k}-P_{k})^{T}(h_{k}-h^{*})}_{\textcircled{4}_{k}}.
\end{aligned}
\end{equation}


We bound the four terms in the right side of  (\ref{eq5.2.1}) separately.

\textbf{Term $\textcircled{1}_{k}$ }: The expectation of $\textcircled{1}_{k}$  never exceeds $[-H,H]$. However, we can not directly utilize this to bound $\textcircled{1}_{k}$. By observing that $\textcircled{1}_{k}$ has a martingale difference structure, we have following lemma based on concentration inequality for martingales.
\begin{lemma}\label{lemma3}
When $T\geq S^{2}AH^{2}\gamma$, with probability $1-3\delta$, it holds that
	$$\sum_{k}\textcircled{1}_{k}\leq KH+ (4H+2\sqrt{12TH})\gamma. $$
\end{lemma}

\textbf{Term $\textcircled{2}_{k}$ }:  Recalling the definition of $V(x,h)$ in Section \ref{algds}, $B_{1,k}^{C}$ implies that 
\begin{equation}\label{eqbd2}
    \textcircled{2}_{k}\leq \sum_{s,a}v_{k,s,a} \bigg (2\sqrt{\frac{V(P_{s,a},h^{*})\gamma}{\max\{N_{k,s,a},1\}}}+2\frac{H\gamma}{\max\{N_{k,s,a},1\}} \bigg)\approx O\bigg(  \sum_{s,a}v_{k,s,a}\sqrt{\frac{V(P_{s,a},h^{*})\gamma}{\max\{N_{k,s,a},1\}}}\bigg),
\end{equation}
where $\approx$ means we omit the insignificant terms. 
We bound RHS of (\ref{eqbd2}) by bounding $\sum_{s,a}N^{(T)}_{s,a}V(P_{s,a},h^{*})$ by $O(TH)$. Formally, we have following lemma.
\begin{lemma}\label{lemmabd2}
	When $T\geq S^{2}AH^{2}\gamma$, with probability $1-\delta$
	$$\sum_{k,s,a}v_{k,s,a}\sqrt{\frac{V(P_{s,a},h^{*})\gamma}{\max\{N_{k,s,a},1\}}}\leq 21\sqrt{SAHT\gamma}.$$
\end{lemma}

\textbf{Term $\textcircled{3}_{k}$ }: According to (\ref{a2c4}) we have
\begin{equation}\label{eqbd3}
\textcircled{3}_{k}\leq \sum_{s,a}v_{k,s,a}L_{2}(\max\{N_{k,s,a},1\},\hat{P}^{(k)}_{s,a},h_{k})\approx O\bigg(  \sum_{s,a}v_{k,s,a}\sqrt{\frac{ V(\hat{P}_{s,a}^{(k)},h_{k})\gamma }{  \max\{N_{k,s,a},1\}  }}\bigg)
\end{equation}
where $L_{2}(N,p,h)=2\sqrt{V(p,h)\gamma/N}+12H\gamma/N+10H\gamma^{3/4}/N^{3/4}$. 
When dealing with the RHS of (\ref{eqbd3}), because $h_{k}$ varies in different episodes, we have to bound the static part and the dynamic part separately. Noting that
\begin{equation}\label{eqbd5}
\begin{aligned}
\sqrt{V(\hat{P}_{s,a}^{(k)},h_{k})}-\sqrt{V(P_{s,a},h^*)} &\leq (\sqrt{V(\hat{P}_{s,a}^{(k)},h_{k})}-\sqrt{V(\hat{P}_{s,a}^{(k)},h^*)})+(\sqrt{V(\hat{P}_{s,a}^{(k)},h^*)}-\sqrt{V(P_{s,a},h^*)})\\& 
\leq \sqrt{|V(\hat{P}_{s,a}^{(k)},h_{k})-V(\hat{P}_{s,a}^{(k)},h^*) | }+\sqrt{| V(\hat{P}_{s,a}^{(k)},h^*)- V(P_{s,a},h^*)|}\\&
\leq \sqrt{4H\sum_{s'}\hat{P}^{(k)}_{s,a,s'}|\delta_{k,s,s'}-\delta^*_{s,s'}| }+\sqrt{4H^{2}|\hat{P}^{(k)}_{s,a}-P_{s,a}|_{1}}\\&
\leq \sum_{s'}\sqrt{4H\hat{P}^{(k)}_{s,a,s'}|\delta_{k,s,s'}-\delta^*_{s,s'}| }
+\sqrt{4H^{2}\sqrt{ \frac{14S\gamma}{    \max\{N_{k,s,a},1\}  } }}\\&
\approx O\Big( \sum_{s'}\sqrt{4H\hat{P}^{(k)}_{s,a,s'}|\delta_{k,s,s'}-\delta^*_{s,s'}| } \Big),
\end{aligned}
\end{equation}
According to the bound of the second term, it suffices to bound
\begin{equation}\label{neq3}
\sqrt{H}\sum_{k,s,a}v_{k,s,a}\sum_{s'}\sqrt{\frac{\hat{P}_{s,a,s'}^{(k)} |\delta_{k,s,s'}-\delta^*_{s,s'}| } {\max\{N_{k,s,a},1 \}}}
\end{equation}
Surprisingly, we find that this term is an upper bound for the fourth term.

\textbf{Term $\textcircled{4}_{k}$ }: Recalling that $\delta^{*}_{s,s'}=h^*_{s}-h^*_{s'}$, according to $B_{2,k}^{C}$ the fourth term can be bounded by: 
\begin{equation}\label{eqbd4}
\begin{aligned}
    \textcircled{4}_{k} &=\sum_{s,a}v_{k,s,a}(\hat{P}^{(k)}_{s,a}-P_{s,a})^{T}(h_{k}-h_{k,s}\textbf{1}-h^{*}+h^*_{s}\textbf{1})=\sum_{s,a}v_{k,s,a}\sum_{s'}(\hat{P}^{(k)}_{s,a,s'}-P_{s,a,s})(\delta^*_{s,s'}-\delta_{k,s,s'}) \\& 
    \approx O\bigg( \sum_{s,a}v_{k,s,a}\sum_{s'}\sqrt{\frac{\hat{P}^{(k)}_{s,a,s'}\gamma}{\max\{N_{k,s,a} ,1\}}}|\delta_{k,s,s'}-\delta^{*}_{s,s'}| \bigg)\\&= O\bigg( \sqrt{H}\sum_{s,a}v_{k,s,a}\sum_{s'}\sqrt{\frac{\hat{P}^{(k)}_{s,a,s'}\gamma|\delta_{k,s,s'}-\delta^{*}_{s,s'}| }{\max\{N_{k,s,a} ,1\}}}\bigg).
\end{aligned}
\end{equation}


To bound  (\ref{neq3}, according to (\ref{hbound}) and the fact $v_{k,s,a}\leq \max\{N_{k,s,a},1\} $ we have  $v_{k,s,a}\sqrt{\frac{\hat{P}^{(k)}_{s,a,s'}|\delta_{k,s,s'}-\delta^*_{s,s'}|}   { \max\{N_{k,s,a},1\}   }}\leq \sqrt{\max\{N_{k,s,a},1\}\hat{P}^{(k)}_{s,a,s'}|\delta_{k,s,s'}-\delta^{*}_{s,s'}|}=\tilde{O}(T^{\frac{1}{4}})$. To be rigorous, we have following lemma.
\begin{lemma}\label{lemma5}  With probability $1-S^{2}T\delta$, it holds that
	\begin{equation}
	\begin{aligned}
	&\sum_{k}\sum_{s,a}v_{k,s,a}\sum_{s'}\sqrt{\frac{\hat{P}^{(k)}_{s,a,s'}|(\delta_{k,s,s'}-\delta_{s,s'}^{*})|}{\max\{N_{k,s,a},1\}}}\leq 11KS^{\frac{5}{2}}A^{\frac{1}{4}}H^{\frac{1}{2}}T^{\frac{1}{4}}\gamma^{\frac{1}{4}}.
	\end{aligned}
	\end{equation}
\end{lemma}
Due to the lack of space, the proofs are delayed to the appendix.

Putting (\ref{eq5.2.1})-(\ref{eqbd5}), (\ref{eqbd4}), Lemma \ref{lemma3}, Lemma \ref{lemmabd2} and Lemma \ref{lemma5} together, we conclude that $\mathcal{R}(T)=\tilde{O}(\sqrt{SAHT})$.

\section{Conclusion}
In this paper we answer the open problems proposed by \citet{jiang2018open} partly by designing an OFU based algorithm EBF and proving a regret bound of $\tilde{O}(\sqrt{HSAT})$ whenever $H$, an upper bound on $sp(h^{*})$ is known. We evaluate state-pair difference of the optimal bias function during learning process. Based on this evaluation, we design a delicate confidence set to guide the agent to explore in the right direction. We also prove a regret bound of $\tilde{O}(\sqrt{DSAT})$ without prior knowledge about $sp(h^{*})$. Both two regret bounds match the corresponding lower bound up to a logarithmic factor and outperform the best previous known bound by an $\sqrt{S}$ factor.

\section*{Acknowledgments}
The authors would like to thank the anonymous reviewers for valuable comments and advice.

\bibliography{reference}
\bibliographystyle{plainnat}

\newpage 
\appendix
\appendixpage
\renewcommand{\appendixname}{Appendix~\Alph{section}}
\setlength{\parindent}{0pt}
\setlength{\parskip}{0.2\baselineskip}
\textbf{Organization.} In Section \ref{A}, we analysis the issues in the proof of [Agrawal $\&$ Jia, 2017].
 In Section \ref{B}, we give some basic lemmas (mainly concentration inequalities).
 Section \ref{C} is devoted to the missing proofs in the analysis of Theorem 1. At last, we present the proof of Corollary 1 in Section \ref{D}.

\section{Mistake in the Analysis of Previous Work}\label{A}
In this section we mainly analysis the mistake in the proof of Lemma C.2 and Lemma C.1 [Agrawal $\&$ Jia, 2017]. The lemma can be described as
\begin{lemma}[Lemma C.2, Agrawal $\&$ Jia, 2017]    \label{AJ1}
	Let $\hat{p}$ be the average of $n$ independent multinoulli trials with parameter $p\in \Delta^{S}$. Let 
	$$Z:=\mathop{\max}\limits_{v\in [0,D]^{S}}(\hat{p}-p)^{T}v.$$
	Then $Z\leq D\sqrt{\frac{2\log(1/\rho)}{n}}$, with probability $1-\rho$. 
\end{lemma}
\noindent We give a counter example as following. Suppose $D=2$,  $p_{i}=\frac{1}{S}$ for each $1\leq i\leq S$, then we have $Z=\mathop{\max}\limits_{v\in [0,2]^{S}}(\hat{p}-p)^{T}v=\mathop{\max}\limits_{v\in [0,2]^{S}}(\hat{p}-p)^{T}(v-\textbf{1})=\mathop{\max}\limits_{v\in [-1,1]^{S}}(\hat{p}-p)^{T}v=\sum_{i=1}^{S}|\hat{p}_{i}-\frac{1}{S}|$, and $\mathbb{E}[Z]=\sum_{i=1}^{S}\mathbb{E}[|\hat{p}_{i}-\frac{1}{S}|]=S\mathbb{E}[|\hat{p}_{1}-\frac{1}{S}|]$ due to symmetry of $p$. Therefore, $\mathbb{E}[Z]= S\mathbb{E}[|\hat{p}_{1}-\frac{1}{S}|]\geq (1-\frac{1}{S})^{n}$. On the other hand, if Lemma \ref{AJ1} is right, by setting $\rho=\frac{1}{n}$ we have
$\mathbb{E}[Z]\leq \sqrt{\frac{2\log(n)}{n}}+\frac{1}{n}$. Letting $S\to \infty$, it follows that $1=\mathop{lim}\limits_{S\to \infty}(1-\frac{1}{S})^{n}\leq 2\sqrt{\frac{2\log(n)}{n}}+\frac{2}{n}$, which is wrong when $n\geq 30$.

\begin{lemma}[Lemma C.1 [Agrawal $\&$ Jia, 2017]]\label{AJ2}
	Let $\tilde{p}\sim Dirichlet(m\overline{p})$. Let
	$$Z:=\mathop{\max}\limits_{v\in [0,D]^{S}}(\tilde{p}-\overline{p})^{T}v.$$
	Then, $Z\leq D\sqrt{\frac{2\log(2/\rho)}{m}}$, with probability $1-\rho$.
\end{lemma}
\noindent Again, to build a counter example, let $D=2$, $\overline{p}_{i}=\frac{1}{S}$ for any $i$. $\mathbb{E}[Z]=S\mathbb{E}[|\tilde{p}_{1}-\frac{1}{S}|]\geq \frac{1}{2}(\mathbb{P}(\tilde{p}_{1}<\frac{1}{2S})+\mathbb{P}(\tilde{p}_{1}>\frac{3}{2S}))$. Note that $\tilde{p}_{1}\sim Beta(\frac{m}{S},m-\frac{m}{S})$. When $m>1$ and $S>m$, the density function of $\tilde{p}_{1}$ is $\frac{x^{\frac{m}{S}-1}(1-x)^{m-\frac{m}{S}}}{B(\frac{m}{S},m-\frac{m}{S})}$ for $x\in (0,1)$, which is decreasing in $x$. Therefore, we have that $\mathbb{P}(\tilde{p}_{1}<\frac{1}{2S})\geq \frac{1}{2} \mathbb{P}(\frac{1}{2S}\leq \tilde{p}_{1}\leq \frac{3}{2S})=\frac{1}{2}(1-(\mathbb{P}(\tilde{p}_{1}<\frac{1}{2S})+\mathbb{P}(\tilde{p}_{1}>\frac{3}{2S})))$, and thus $\mathbb{P}(\tilde{p}_{1}<\frac{1}{2S})+\mathbb{P}(\tilde{p}_{1}>\frac{3}{2S})\geq \frac{1}{3}$. As a result, $\mathbb{E}[Z]\geq \frac{1}{6}$, which contradicts to Lemma \ref{AJ2}. 
Moreover, we find that the mistake in their proof lies in the derivation
\begin{equation*}
\begin{aligned}
\mathbb{E}[DY-Z|Z=z:z\in \mathcal{E}_{v}]&=\mathbb{E}[DY-D\mathbb{E}[Y_{v}-Z|Z=z:z\in \mathcal{E}_{v}]\\
&=\mathbb{E}[DY_{v}-D\mathbb{E}[Y_{v}]-(\hat{p}-p)^{T}v|(\hat{p}-p)^{T}v]\\
&=\mathbb{E}[DY_{v}-\hat{p}^{T}v|\hat{p}^{T}v]=0
\end{aligned}
\end{equation*}
Actually, $\{Z=z:z\in \mathcal{E}_{v}\}\subsetneqq \{Z=z:z=(\hat{p}-p)^{T}v\}$ because given the value of $Z=z$, it's still unknown that which $v$ is selected to maximize $(\hat{p}-p)^{T}v$. More rigorously, we have
$\mathbb{E}[\mathbb{E}[DY_v-\hat{p}^Tv|Z=z,z \in \mathcal{E}_{v}]|Z \in \mathcal{E}_{v}]=\mathbb{E}[DY_{v}-\hat{p}^Tv|Z \in \mathcal{E}_v]=p^Tv-\mathbb{E}[\hat{p}^Tv|Z \in \mathcal{E}_v]<0$, since $(\hat{p}-p)^Tv>0$ conditioning on $Z$ in $\mathcal{E}_{v}$ (except for $\hat{p}=p$). This contradicts to the analysis of Lemma C.2 in [Agrawal $\&$ Jia, 2017], which says that $\mathbb{E}[DY_{v}-\hat{p}^Tv|Z=z,z \in \mathcal{E}_{v}]=0$.

Therefore, the algorithm in [Agrawal $\&$ Jia, 2017] may not reach the regret bound of $\tilde{O}(D\sqrt{SAT})$ .

\section{Some Basic Lemmas}\label{B}
In this section, we present some useful lemmas. Some of them are well known so that we omit the proof.
\begin{lemma}[Azuma's Inequality]\label{lemma10} Suppose $\{X_{k}\}_{k=0,1,2,3,..}$ is a martingale and $|X_{k+1}-X_{k}|<c$. Then for all positive integers $N$ and all positive $t$,
\begin{equation}\label{Azuma}
    \mathbb{P}(|X_{N}-X_{0}|\geq t)\leq 2exp(\frac{-t^{2}}{2Nc^{2}}).
\end{equation}
Let $t=c\sqrt{2N\log(2/\delta)}$, then $\mathbb{P}(|X_{N}-X_{0}|\geq t)\leq \delta$.
\end{lemma}

\begin{lemma}[Bernstein Inequality] \label{lemma11}Let $\{X_{k}\}_{k\geq 1}$ be independent zero-mean random variables. Suppose that $|X_{k}|\leq M$ for all $k$. Then, for all positive $t$
\begin{equation}\label{Bernstein}
    \mathbb{P}(|\sum_{k=1}^{n}X_{k}|\geq t)\leq 2exp(-\frac{t^{2}}{2(\sum_{k=1}^{n}\mathop{E}[X_{k}^{2}]+\frac{1}{3}Mt)}).
\end{equation}
Let $t=2\sqrt{\sum_{k=1}^{n}\mathbb{E}[X_{k}^{2}]\log(2/\delta)}+2M\log(2/\delta)$, then $\mathbb{P}(|\sum_{k=1}^{n}X_{k}|\geq t)\leq \delta$.
\end{lemma}
\begin{lemma}\label{originallemma2}
	 Let $\hat{p}_{n}$ be the average of $n$ independent multinomial trials with parameter $p\in \Delta^{m}$. Then, for any fixed vector $u\in \mathbb{R}^{m}$, with probability   $1-\delta$, it holds that
	$$|(\hat{p}_{n}-p)^{T}u|\leq 2\sqrt{\frac{V(p,u)\gamma}{n}}+2\frac{sp(u)\gamma}{n}.$$
\end{lemma}
\begin{proof}
	 Given $u\in \mathbb{R}^{m}$ and $p\in \Delta^{m}$, let $\{X_{k}\}_{k\geq 1}$ be i.i.d. random variable s.t. $\mathbb{P}(X_{k}=u_{i}-p^{T}u)=p_{i}$, $\forall k$. Because $E[X_{k}^{2}]=V(p,u)$ and $\frac{1}{n}\sum_{k=1}^{n}X_{k}=(\hat{p}_{n}-p)^{T}u$, according to Lemma \ref{lemma11} we get that
\begin{equation*}
    \mathbb{P}(|(\hat{p}_{n}-p)^{T}u|\geq 2\sqrt{\frac{V(p,u)\gamma}{n}}+2\frac{sp(u)\gamma}{n})\leq \delta.
\end{equation*}
\end{proof}

\begin{lemma}[Freedman (1975)]\label{lemma12}
	Let $(M_{n})_{n\geq 0}$ be a  martingale such that $M_{0}=0$. Let $V_{n}=\sum_{k=1}^{n}\mathbb{E}[(M_{k}-M_{k-1})^{2}|\mathcal{F}_{k-1}]$ for $n\geq 0$,
	where $\mathcal{F}_{k}=\sigma(M_{1},M_{2},...,M_{k})$. Then, for any positive $x$ and for any positive $y$,
	\begin{equation}\label{Bernstein2}
	\mathbb{P}(M_{n}\geq nx,V_{n}\leq ny)\leq exp(-\frac{nx^{2}}{2(y+\frac{1}{3}x)}).
	\end{equation}
\end{lemma}

\begin{lemma}\label{lemma13} Suppose $M$ is a flat MDP. Let $h$ and $\rho$ denote the optimal bias function and the optimal average reward respectively. We run $N$ steps under $M$ and get a trajectory $L$ of length $N$. Then we have, no matter which action is chosen in each step, for each $n\in [N]$, with probability $1-\delta$, it holds that
\begin{equation}
    |\sum_{i=1}^{n}(r_{i}-\rho)|\leq
(2\sqrt{n\gamma}+1)sp(h).\label{3.1}
\end{equation}
Moreover, suppose that the reward is bounded in $[0,1]$, $n\geq 4\gamma sp(h)^{2}$ and $sp(h)\geq 10$, then  with probability $1-2\delta$ it holds that
\begin{equation}
    |\sum_{i=1}^{n}(r_{i}-\rho)|\leq 4\sqrt{n\gamma sp(h)}+sp(h).\label{3.2}
\end{equation}
\end{lemma}
\begin{proof}
Let $M_{0}=h_{s_{1}}$ and $M_{n}-M_{n-1}=h_{s_{n+1}}-h_{s_{n}}+r_{n}-\rho$ for $n\geq 1$. Then
 $\{M_{n}-M_{0}\}_{n\geq 0}$ is a martingale martingale difference sequence since $\mathbb{E}[h_{s_{n+1}}-h_{s_{n}}+r_{n}-\rho|\mathcal{F}_{n-1}]=\sum_{a}\mathbb{P}(a_{t}=a)\mathbb[E][h_{s_{n+1}}-h_{s_{n}}+r_{n}-\rho|\mathcal{F}_{n-1},a_{t}=a]=\sum_{a}\mathbb{P}(a_{t}=a)(P_{s_{n},a}^{T}h-h_{s_{n}}+r_{s_{n},a}-\rho)=0$. Because $|M_{n}-M_{n-1}|\leq  \max_{a}|P_{s_{n},a}^{T}h-h_{s_{n+1}}| \leq sp(h)$, $V_{n}\leq nsp(h)^{2}$. Plug $y=sp(h)^{2}$ and $x = \frac{2\sqrt{\gamma}sp(h)}{\sqrt{n}}$ into  (\ref{Bernstein2}), then  (\ref{3.1}) follows easily. To prove  (\ref{3.2}), we need to provide a tighter bound for $V_{n}$. For $v\in \mathbb{R}^{S}$, we use $v^{2}$ to denote the vector $[v_{1}^{2},v_{2}^{2},...,v_{S}^{2}]^{T}$. Because $V_{n}=\sum_{k=1}^{n}\mathbb{E}[(M_{k}-M_{k-1})^{2}|\mathcal{F}_{k-1}]=\sum_{k=1}^{n}P_{s_{k},a_{k}}^{T}h^{2}-(P_{s_{k},a_{k}}^{T}h)^{2}$ and $P^{T}_{s_{k},a_{k}}h-h_{s_{k}}=\rho-r_{s_{k},a_{k}}$, we have that
\begin{equation*}
    V_{n}\leq \sum_{k=1}^{n}(P_{s_{k},a_{k}}^{T}h^{2}-h_{s_{k}}^{2})+\sum_{k=1}^{n}(sp(h)|\rho-r_{s_{k},a_{k}}|+(\rho-r_{s_{k},a_{k}})^{2}).
\end{equation*}

By the assumption the reward is bounded in $[0,1]$, we have $\rho\in [0,1]$ and $|\rho-r_{s_{k},a_{k}}|\leq 1$. Let $X_{n}=\sum_{k=1}^{n}(P^{T}_{s_{k},a_{k}}h^{2}-h_{s_{k+1}}^{2})=V_{n}+h^{2}_{s_{n+1}}-h^{2}_{s_{1}}$ for $n\geq 1$ and $X_{0}=0$. It's clear $\{X_{n}\}_{n\geq 0}$ is a martingale difference sequence and $|X_{k}-X_{k-1}|\leq sp(h)^{2}$. According to Lemma \ref{lemma10}, we have that
\begin{equation*}
   P(|X_{n}|\geq \sqrt{2n\gamma}sp(h)^{2})\leq \delta 
\end{equation*}
Then it follows that with probability $1-\delta$, $|V_{n}|\leq (\sqrt{2n\gamma}+1)sp(h)^{2}+n(2sp(h)+1)$. When $n\geq 4\gamma sp(h)^{2}$ and $sp(h)\geq 10$, we get $|V_{n}|\leq 4nsp(h)$. Again, plugging $x = \frac{4\sqrt{\gamma sp(h)}}{\sqrt{n}}$ and $y=4sp(h)$ into  (\ref{Bernstein2}), noticing that $n\geq 16\gamma sp(h)$, we conclude that, with probability $1-2\delta$, $|\sum_{i=1}^{n}(r_{i}-\rho)|\leq 4\sqrt{n\gamma sp(h)}+sp(h)$.
\end{proof}

We introduce a technical lemma which is actually an expansion of Lemma 19, [Jaksch et al., 2010].
\begin{lemma}\label{lemma14}
Suppose $\{x_{n}\}_{n=1}^{N}$ is sequence of positive real number with $x_{1}=1$ and $x_{n}\leq \sum_{i=1}^{n-1}x_{i}$ for $n = 2,3,...,N-1$. Then we have, for any $0<\alpha <1$,
$$x_{1}+\sum_{n=2}^{N}x_{n}(\sum_{i=1}^{n-1}x_{i})^{-\alpha}\leq \frac{2^{\alpha}}{1-\alpha}(\sum_{n=1}^{N}x_{n})^{1-\alpha}.$$
Moreover, in the case $\alpha = 1$, we have
$$x_{1}+\sum_{n=2}^{N}x_{n}(\sum_{i=1}^{n-1}x_{i})^{-1}\leq 1+2\log(\sum_{n=1}^{N}x_{n}).$$
\end{lemma}
\begin{proof}
Let $S_{n}=\sum_{1\leq i\leq n}x_{i}$ for $n\geq 1$, then it follows $2S_{n}\geq S_{n+1}$ for $n\in [N-1]$. By basic calculus, when $\alpha<1$, for $n\geq 2$ we have
$$S_{n}^{1-\alpha}-S_{n-1}^{1-\alpha}\geq (1-\alpha)x_{n}S_{n}^{-\alpha}\geq
\frac{1-\alpha}{2^{\alpha}}x_{n}S_{n-1}^{-\alpha}.$$
Note that $S_{1}^{1-\alpha}=1$, we then have $x_{1}+\sum_{n=2}^{N}x_{n}S_{n-1}^{-\alpha}\leq 1+ \frac{2^{\alpha}}{1-\alpha}\sum_{n=2}^{N}(S_{n}^{1-\alpha}-S_{n-1}^{1-\alpha})\leq \frac{2^{\alpha}}{1-\alpha}S_{N}^{1-\alpha}+1-\frac{2^{\alpha}}{1-\alpha}\leq  \frac{2^{\alpha}}{1-\alpha}S_{N}^{1-\alpha}$.

In the case $\alpha = 1$, for $n\geq 2$ we have
$$\log(S_{n})-\log(S_{n-1})\geq \frac{x_{n}}{S_{n}}\geq \frac{x_{n}}{2S_{n-1}}. $$
Note that $\log(S_{1})=0$, we then have $x_{1}+\sum_{n=2}^{N}x_{n}S_{n-1}^{-1}\leq 1+2(\log(S_{n}-\log(S_{1})))=1+2\log(S_{n})$.
\end{proof}
Applying Lemma \ref{lemma14} to $\{v_{k,s,a}\}_{k\geq 1}$, we have that for any $0<\alpha<1$
\begin{equation*}
    \sum_{k}\frac{v_{k,s,a}}{\max\{N_{k,s,a},1\}^{\alpha}}\leq \frac{2^{\alpha}}{1-\alpha}(N^{(T)}_{s,a})^{1-\alpha} 
\end{equation*}
Combining this inequality and Jenson's inequality, we get that
\begin{equation}
    \sum_{k,s,a}\frac{v_{k,s,a}}{\max\{N_{k,s,a},1\}^{\alpha}}\leq \frac{2^{\alpha}}{1-\alpha}SA(\frac{T}{SA})^{1-\alpha} \label{4.1}
\end{equation}
In the case $\alpha = 1$, we also have
\begin{equation}
    \sum_{k,s,a}\frac{v_{k,s,a}}{\max\{N_{k,s,a},1\}}\leq SA+2SA\log(\frac{T}{SA}) \label{4.2}
\end{equation}
With a slightly abuse of notations, we use $N_{k,s,a}$ to denote $\max\{N_{k,s,a},1\}$ in the rest of the paper for simplicity.

\section{Missing Proofs in the Analysis of Theorem 1}\label{C}
In this section, we present the proofs of Lemma 1-5 and give a detailed proof of Theorem 1.
\subsection{Proof of Lemma 1}\label{C.1}
Let $h\in \mathbb{R}^{S}$ and $\rho\in \mathbb{R}$ be fixed. We define a Markov process $X$ with state space $\mathcal{S}$. Let $\{\mathcal{F}_{t} \}_{t\geq 1}$ be the corresponding filtered algebra, i.e., $\mathcal{F}_{t}=\sigma(X_{1},...,X_{t})$. Let $s_{1}$ be the initial state. For each state $s$, there are some actions and each action $a$ is equipped with a transition probability vector $p_{s,a}$ and a reward $r'_{s,a}=h_{s}+\rho-p_{s,a}^{T}h$.  In the $t$-th step, there is a policy $\pi_{t}$. We select an action according to $\pi_{t}$, then execute it and reach the next state. We then have
$\mathbb{P}[p_{t}=p_{s_{t},a},r'_{t}=r'_{s_{t},a}]=\pi_{t,a}$, where $p_{t}$ is transition probability and $r'_{t}$ is the reward in current step. 

Then it is clear $\{(s_{t},s_{t+1},r'_{t})\}_{t=1}^{n}$ is measurable with respect to  $\mathcal{F}_{n}$. 
For any two different states $s,s'\in \mathcal{S}$, given a trajectory $L=\{(s_{t},s_{t+1},r'_{t})\}_{t=1}^{n}$, we define an indicator function $I_{s,s'}(L,t)$ as following:

If $t\geq n+1$, $I_{s,s'}(L,t)=0$. Otherwise, let $U=\{i|s_{i}\in \{s,s'\},1\leq i\leq t \}$.  If $U$ is empty, $I_{s,s'}(L,t)=0$; else $I_{s,s'}(L,t)=\mathbb{I}[s_{i^*}=s]$ where  $i^{*}$ be the maximal element of $U$ .

Let $L$ be the $N$-step trajectory of $X$ and $I_{s,s'}(t)=I_{s,s'}(L,t)$.  Note that $I_{s,s'}(t)$ is a random variable, and it only depends on $\{s_{u}\}_{u=1}^{t}$, which is measurable with respect to $\mathcal{F}_{t-1}$. Let $W_{t}=\sum_{u=1}^{t}I_{s,s'}(u)(r_{u}-h_{s_{u}}+h_{s_{u+1}}-\rho)$, then we have $\mathbb{E}[W_{1}]=0$ and $\mathbb{E}[W_{t}-W_{t-1}|\mathcal{F}_{t-1}]=0$ for $t\geq 2$. It follows that $\{W_{t}\}_{t=1}^{N}$ is a martingale with respect to $\{\mathcal{F}_{t} \}_{t=1}^{N}$. Because $|W_{t}-W_{t-1}|=|I_{s,s'}(t)(r'_{t}-h_{s_{t}}+h_{s_{t+1}}-\rho^*)|\leq  \max_{a}|I_{s,s'}(t)(h_{s_{t+1}}-p_{s{_t},a}^{T}h   )|\leq sp(h)$ and $|W_{1}|\leq sp(h)$, by  (\ref{Azuma}), we have that, for any $n\leq N$, 
$$\mathbb{P}(|W_{n}|\geq \sqrt{2N\gamma}sp(h)+sp(h))\leq \delta.$$
Then it follows that, with probability $1-N\delta$, for any $n\in [N]$, 
$$|W_{n}|\leq \sqrt{2N\gamma}sp(h)+sp(h).$$
Recall the notations in Definition 4, $ts_{1}(\mathcal{L}):=\min\{\min\{t|s_{t}=s\},N+2\}$,
$$te_{k}(\mathcal{L}):=\min\{\min\{t|s_{t}=s',t>ts_{k}(\mathcal{L})\},N+2\}, k\geq 1,$$
$$ts_{k}(\mathcal{L}):=\min\{\min\{t|s_{t}=s,t>te_{k-1}(\mathcal{L})\},N+2\}, k\geq 2.$$
and $c(s,s',\mathcal{L}):=\max\{k|te_{k}(\mathcal{L})\leq N+1\}.$
According to the definition of $I_{s,s'}(t)$, for any $c\in [c(s,s',\mathcal{L})]$, we have 
$$W_{te_{c}(\mathcal{L})-1}=\sum_{u=1}^{c}(\sum_{ts_{u}(\mathcal{L})\leq t\leq te_{u}(\mathcal{L})-1}(r'_{t}-\rho) +  h_{s'}-h_{s}).$$ 
Given an algorithm $\mathcal{G}$,
we can view $\mathcal{G}$ as a function which maps previous samples, policies and current state to a policy in current state,
and we use $\mathcal{G}_{t}:=\mathcal{G}(s_{t},(s_{u},\pi_{u},a_{u},r_{u},s_{u+1})_{u=1}^{t-1})$ to denote this policy. By setting $h=h^{*}$, $\rho =\rho^{*}$, $p_{s,a}=P_{s,a}$ and $\pi_{t}=\mathcal{G}_{t}$, 
we have  $r_{s,a}=h^{*}_{s}+\rho^{*}-p_{s,a}^{T}h^{*}=r'_{s,a}$, since $M$ is flat. 
It then follows that $$W_{te_{c}(\mathcal{L})-1}=\sum_{u=1}^{c}(\sum_{ts_{u}(\mathcal{L})\leq t\leq te_{u}(\mathcal{L})-1}(r_{t}-\rho^{*}) +  h_{s'}-h_{s}).$$ 
As we proved before, with probability $1-N\delta$, it holds that for any $1\leq n\leq N$,
$$|W_{n}|\leq \sqrt{2N\gamma}sp(h)+sp(h).$$
Because $1\leq ts_{c}(\mathcal{L}) \leq te_{c}(\mathcal{L})-1\leq N$ for any $1\leq c\leq c(s,s',\mathcal{L}) $, Lemma 1 follows easily.

\subsection{ Proof of Lemma 2}\label{C.2}
Recall the definition of bad events.
\begin{equation*}
\begin{aligned}
&B_{1,k}:=\bigg\{\exists (s,a), s.t. |(P_{s,a}-\hat{P}^{(k)}_{s,a})^{T}h^{*}|> 2\sqrt{\frac{V(P_{s,a},h^{*})\gamma)}{N_{k,s,a}}}+2\frac{sp(h^{*}\gamma)}{N_{k,s,a}}  \bigg\},\\
&B_{2,k}=\bigg\{\exists (s,a,s'), s.t. |\hat{P}^{(k)}_{s,a,s'}-P_{s,a,s'}|> 2\sqrt{\frac{\hat{P}^{(k)}_{s,a,s'}\gamma}{N_{k,s,a}}}+\frac{3\gamma}{N_{k,s,a}}+\frac{4\gamma^{\frac{3}{4}}}{N_{k,s,a}^{\frac{3}{4}}}  \bigg\},\\
& B_{3,k}=\Big\{|\sum_{1\leq t<t_{k}}(\rho^{*}-r_{s_{t},a_{t}})|> 26HS\sqrt{AT\gamma},\sum_{k'<k}\sum_{s,a}v_{k',s,a}reg_{s,a}> 22HS\sqrt{AT\gamma}  \Big\}\\
&B_{4,k}=\big\{\{(\pi^{*},P^{*},h^{*},\rho^{*})|\pi^{*} \mbox{is a deterministic optimal policy}\}\cap \mathcal{M}_{k}=\varnothing  \big\},
\end{aligned}
\end{equation*}
$B_{k} =  B_{1,k}\cup B_{2,k}\cup B_{3,k}\cup B_{4,k}$ and $B=\cup_{1\leq k\leq K+1}B_{k}$.

It's easy to see that for each $k$, $B_{1,k}$ and $B_{2,k}$ indicate the events where the concentration inequalities fail, and thus have a small probability. Suppose $B_{k'}^{C}$ occurs for each $k'< k$, we get that the regret before the $k$-th episode does not exceed $\tilde{O}(HS\sqrt{AT})$ with high probability based on the analysis of R\scriptsize{EGAL}\normalsize.C. 

To show $\mathbb{P}(B_{4,k})$ is small, we prove that, conditioned on $\cap_{1\leq k'<k}B_{k'}^{C}$ occurs, with high probability, it holds that $h^{*}\in \mathcal{H}$. Let $\pi^{*}$ be a deterministic optimal policy. Note that if (\ref{a2c2})-(\ref{a2c4}) holds for any $s,a,s'$ with $P'(\pi)=P$ where $P$ is the true transition model, we then have $(\pi^{*},P,h^{*},\rho^{*})\in \mathcal{M}_{k}$, since (\ref{a2c5}) holds due to the optimality of $\pi^{*}$.
Putting all together, we can bound $\mathbb{P}(B)$ up to $\tilde{O}(S^{3}A^{2}T)\delta$.

Note that $t_{K+1}-1=T$, then $B_{K+1}$ is also well defined. Firstly, for each $ k$, according to Lemma \ref{originallemma2}, we have $\mathbb{P}(B_{1,k})\leq SA\delta$ directly.

To bound the probability of $B_{2,k}$, let $(s,a)$ be fixed. Defining $g(x)=[x,1-x]^{T}$ for $x\in [0,1]$. Then we have $|x_{1}-x_{2}|=\frac{1}{2}|g(x_{1})-g(x_{2})|_{1}=\frac{1}{2}\mathop{sup}\limits_{y\in \{-1,1\}^{2}}(g(x_{1})-g(x_{2}))^{T}y$ for $x_{1},x_{2}\in [0,1]$. It follows that $\mathbb{P}(|x_{1}-x_{2}|\geq 2\epsilon)\leq 4\mathop{sup}\limits_{y\in \{-1,1\}^{2}}\mathbb{P}((g(x_{1})-g(x_{2}))^{T}y\geq \epsilon)$. Noting that $V(g(x),y)\leq 4x$ for each $y\in \{-1,1\}^{2}$, according to Lemma \ref{originallemma2} we have, for any $y\in \{-1,1\}^{2}$
$$\mathbb{P}(|(g(\hat{P}^{(k)}_{s,a,s'})-g(P_{s,a,s'}))^{T}y|\geq 2\sqrt{\frac{4P_{s,a,s'}\gamma}{N_{k,s,a}}}+\frac{2\gamma}{N_{k,s,a}} )\leq \delta$$
which means that $\mathbb{P}(|\hat{P}^{(k)}_{s,a,s'}-P_{s,a,s'}|\geq 2\sqrt{\frac{P_{s,a,s'}\gamma}{N_{k,s,a}}}+\frac{\gamma}{N_{k,s,a}})\leq 4\delta$. Suppose that the event $\{|\hat{P}^{(k)}_{s,a,s'}-P_{s,a,s'}|< 2\sqrt{\frac{P_{s,a,s'}\gamma}{N_{k,s,a}}}+\frac{\gamma}{N_{k,s,a}}\}$ occurs, then we have
\begin{equation*}
    \begin{aligned}
    |\hat{P}^{(k)}_{s,a,s'}-P_{s,a,s'}|&\leq 2\sqrt{\frac{P_{s,a,s'}\gamma}{N_{k,s,a}}}+\frac{\gamma}{N_{k,s,a}}\\
    & \leq 2\sqrt{\frac{(\hat{P}^{(k)}_{s,a,s'}+2\sqrt{\frac{\gamma}{N_{k,s,a}}}+\frac{\gamma}{N_{k,s,a}})\gamma}{N_{k,s,a}}}+\frac{\gamma}{N_{k,s,a}}\\
    &\leq 2\sqrt{\frac{\hat{P}^{(k)}_{s,a,s'}\gamma}{N_{k,s,a}}}+\frac{3\gamma}{N_{k,s,a}}+\frac{4\gamma^{\frac{3}{4}}}{N_{k,s,a}^{\frac{3}{4}}}.
    \end{aligned}
\end{equation*}
Therefore, $\mathbb{P}(B_{2,k})\leq 4S^{2}A\delta$.

For $k=1$, $B_{3,k}^{C}$ and $B_{4,k}^{C}$ holds trivially. For $k>1$, assuming $\cap_{k'\geq1}B_{1,k'}^{C}$, $\cap_{k'\geq1}B_{2,k'}^{C}$, $\cap_{1\leq k'<k}B_{3,k'}^{C}$ and  $\cap_{1\leq k'<k}B_{4,k'}^{C}$ hold. We start to bound $\mathbb{P}(B_{4,k})$.  Note that $B_{3,k-1}^{C}$ ensures that
\begin{equation}
    \sum_{1\leq k'<k}\sum_{s,a}v_{k,s,a}reg_{s,a}\leq 22HS\sqrt{AT\gamma} \label{C.4.1}
\end{equation}
Note that if we replace the reward function $r_{s,a}$ by $r'_{s,a}=r_{s,a}+reg_{s,a}$, the MDP $M$ will be \emph{flat}. According to Lemma 1, we have
\begin{equation}
   |\sum_{i=1}^{c(s,s',\mathcal{L}_{t_{k}-1})}\sum_{ts_{i}\leq j\leq te_{i}-1}(r_{s_{j},a_{j}}+reg_{s_{j},a_{j}}-\rho^{*})-c(s,s',\mathcal{L}_{t_{k}-1})\delta_{s,s'}^{*}|\leq (\sqrt{2 T\gamma}+1)H \label{C.4.2}
\end{equation}
with probability $1-T\delta$. Combining  (\ref{C.4.1}) and  (\ref{C.4.2}), we get that
\begin{equation}
   |\sum_{i=1}^{c(s,s',\mathcal{L}_{t_{k}-1})}\sum_{ts_{i}\leq j\leq te_{i}-1}(r_{s_{j},a_{j}}-\rho^{*})-c(s,s',\mathcal{L}_{t_{k}-1})\delta_{s,s'}^{*}|\leq (\sqrt{2 T\gamma}+1)H+22HS\sqrt{AT\gamma} \label{C.4.3}
\end{equation}
Furthermore, $B_{3,k}^{C}$ also implies that $|\sum_{1\leq k'<k}\sum_{s,a}v_{k,s,a}(\rho^{*}-r_{s,a})|\leq 26HS\sqrt{AT\gamma}$, then it follows $(\sum_{1\leq k'<k}l_{k'})|\hat{\rho}_{k}-\rho^{*}|\leq 26HS\sqrt{AT\gamma}$ where $l_{k'}$ is the length of the $k'$-th episode and $\hat{\rho}_{k}=\frac{\sum_{1\leq  t\leq t_{k}-1}r_{t}}{\max\{\sum_{1\leq k'\leq k}l_{k'},1 \} }$ is the average reward before the $k$-th episode. Therefore, we have that
\begin{equation}\label{C.4.4}
\begin{aligned}
   &|\sum_{i=1}^{c(s,s',\mathcal{L}_{t_{k}-1})}\sum_{ts_{i}\leq j  
   \leq te_{i}-1}(r_{s_{j},a_{j}}-\hat{\rho}_{k})-c(s,s',\mathcal{L}_{t_{k}-1})\delta_{s,s'}^{*}| \\&
    \leq |\sum_{i=1}^{c(s,s',\mathcal{L}_{t_{k}-1})}\sum_{ts_{i}\leq j\leq te_{i}-1}(r_{s_{j},a_{j}}-\rho^{*})-c(s,s',\mathcal{L}_{t_{k}-1})\delta_{s,s'}^{*}| +|(\sum_{1\leq k'<k}l_{k'})(\hat{\rho}_{k}-\rho^{*})| \\& \leq (\sqrt{2 T\gamma}+1)H+48HS\sqrt{AT\gamma} 
\end{aligned}
\end{equation}
which means that $h^{*}\in \mathcal{H}$ in the beginning of the $k$-th episode.   


The last step is to prove that (\ref{a2c2}), (\ref{a2c3}) and (\ref{a2c4}) hold for $P'(\pi)=P$ with high probability. (\ref{a2c2}) holds evidently because of $B_{2,k}^{C}$.
According to the $L_{1}$ norm concentration inequality [Weissman et al,. 2003], we see that $\mathbb{P}(|P_{s,a}-\hat{P}^{(k)}_{s,a}|\leq \sqrt{\frac{12S\gamma}{N_{k,s,a}}})\leq \delta$, thus (\ref{a2c3}) is satisfied. In order to prove (\ref{a2c4}) holds for $P'=P$ with high probability,  by using Lemma \ref{originallemma2} twice, we have that for each $(s,a)$
\begin{equation*}
    \begin{aligned}
    |(P_{s,a}-\hat{P}^{(k)}_{s,a})^{T}h^{*}|&\leq 2\sqrt{\frac{V(P_{s,a},h^{*})\gamma}{N_{k,s,a}}}+2\frac{H\gamma}{N_{k,s,a}}\\
    &\leq 2\sqrt{\frac{V(\hat{P}^{(k)}_{s,a},h^{*})\gamma}{N_{k,s,a}}}+2\sqrt{\frac{|V(P_{s,a},h^{*})-V(\hat{P}^{(k)}_{s,a},h^{*})|\gamma}{N_{k,s,a}}}+2\frac{H\gamma}{N_{k,s,a}}\\
    & \leq 2\sqrt{\frac{V(\hat{P}^{(k)}_{s,a},h^{*})\gamma}{N_{k,s,a}}}+2\sqrt{\frac{H^{2}(2\sqrt{\frac{\gamma}{N_{k,s,a}}}+2\frac{\gamma}{N_{k,s,a}})\gamma}{N_{k,s,a}}}+2\frac{H\gamma}{N_{k,s,a}}\\
    &\leq 2\sqrt{\frac{V(\hat{P}^{(k)}_{s,a},h^{*})\gamma}{N_{k,s,a}}}+12\frac{H\gamma}{N_{k,s,a}}+10\frac{H\gamma^{3/4}}{N_{k,s,a}^{3/4}}.
    \end{aligned}
\end{equation*}
holds with probability $1-2\delta$. Therefore, $\mathbb{P}(B^{C}_{4,k})\leq (T+3SA)\delta$. 

On the other side, note that $\cap_{1\leq k'<k} B_{4,k'}^{C}$ ensures that $\{(\pi^{*},P^{*},h^{*},\rho^{*})|\pi^{*}\in \mathcal{O}\}\cap \mathcal{M}_{k}\neq \varnothing$. It means that $\rho(\pi_{k})\geq \rho^{*}$. Following the proof of Theorem 2 [Bartlett and Tewari, 2009], we get that when $T\geq A\log(T)$
\begin{equation*}
\begin{aligned}
\sum_{1\leq t\leq t_{k}-1}(\rho^{*}-r_{t})&\leq |\sum_{k}v_{k}^{T}(P'_{k}-P_{k})|_{1}H+|\sum_{k}v_{k}^{T}(P_{k}-I)h_{k}|  \\& \leq 2H(\sum_{k,s,a}v_{k,s,a}\sqrt{\frac{12S\gamma}{N_{k,s,a}}}+\sqrt{2T\gamma} +K) \\&
\leq    18HS\sqrt{AT\gamma}
\end{aligned}
\end{equation*}

with probability $1-2AT\delta$. Moreover, note that
\begin{equation}
 \sum_{1\leq t\leq t_{k}-1}reg_{s_{t},a_{t}}=\sum_{1\leq t\leq t_{k}-1}(\rho^{*}-r_{t})+\sum_{1\leq t\leq t_{k}-1}(h^{*}_{s_{t}}-P_{s_{t},a_{t}}^{T}h^{*})   \label{C.4.5}
\end{equation}
By Azuma's inequality (Lemma \ref{lemma10}), we have that
\begin{equation}
 |\sum_{1\leq i\leq t}(h^{*}_{s_{i}}-P_{s_{i},a_{i}}^{T}h^{*})|\leq 2H+\sqrt{2T\gamma}H \label{C.4.6}
\end{equation}
holds for any $1\leq t\leq T$ with probability $1-T\delta$. Assuming  (\ref{C.4.5}) and (\ref{C.4.6}) hold for any $1\leq t\leq T$, noticing that $reg_{s,a}\geq 0$ for any $(s,a)$, we have
\begin{equation*}
    |\sum_{1\leq t\leq t_{k}-1}reg_{s_{t},a_{t}}|\leq 18HS\sqrt{AT\gamma}+2H+\sqrt{2T\gamma}H\leq 22HS\sqrt{AT\gamma} 
\end{equation*}
and
$$|\sum_{1\leq t\leq t_{k}-1}(\rho^{*}-r_{t})|\leq|\sum_{1\leq t\leq t_{k}-1}reg_{s_{t},a_{t}}|+ |\sum_{1\leq i\leq t}(h^{*}_{s_{i}}-P_{s_{i},a_{i}}^{T}h^{*})|\leq 26HS\sqrt{AT\gamma}$$
At last, we conclude that when $\cap_{k'\geq1}B_{1,k'}^{C}$, $\cap_{k'\geq1}B_{2,k'}^{C}$, $\cap_{1\leq k'<k}B_{3,k'}^{C}$ and  $\cap_{1\leq k'<k}B_{4,k'}^{C}$ hold, $\mathbb{P}(B_{3,k})\leq (2AT+T)\delta$.\\
Putting all together we have 
$$\mathbb{P}(B)\leq (K+1)(2AT+8S^{2}A+2T)\delta\leq (6AT+12S^{2}A)SA\log(T)\delta$$
when $T\geq A\log(T)$ and $SA\geq 4$.

\subsection{ Proof of Lemma 3}\label{C.3}
\begin{lemma}\label{L.5.1}
Let $V=\sum_{k}\sum_{s,a}v_{k,s,a}V(P_{s,a},h_{k})$ and $W=\sum_{k}\textcircled{1}_{k}$. For any $C>0$, we have
$$\mathbb{P}(|V|\leq C, |W|\geq KH+(4H +2\sqrt{C})\gamma)\leq 2\delta$$ 
\end{lemma}
\begin{proof} Let $X_{k,n}=\sum_{i=1}^{n}(P_{s_{k,i},a_{k,i}}^{T}h_{k}-h_{k,s_{k,i+1}})$ where $(s_{k_{i}},a_{k_{i}},r_{k_{i}},s_{k_{i+1}})$ is the $i$-th sample in the $k$-th episode. We use $l_{k}$ to denote the length of the $k$-th episode. Let $e_{n}=\max\{k|t_{k}\leq n\}$ and $Z_{n} = \sum_{k=1}^{e_{n}-1}X_{k,l_{k}}+X_{e_{n},n-t_{e_{n}}+1}$. Let $\mathcal{F}_{n}=\sigma(Z_{1},...,Z_{n})$.  It's easy to see $E[Z_{n+1}-Z_{n}|\mathcal{F}_{n}]=E[X_{e_{n},n+2-t_{e_{n}}}-X_{e_{n},n+1-t_{e_{n}}}|\mathcal{F}_{n}]=0$ if $e_{n}=e_{n+1}$, and  $E[Z_{n+1}-Z_{n}|\mathcal{F}_{n}]=E[X_{e_{n+1},1}|\mathcal{F}_{n}]=0$ otherwise. Therefore, $\{Z_{n}\}_{n\geq 1}$ is a martingale with respect to $\{\mathcal{F}_{n}\}_{n\geq 1}$. On the other hand, it's easy to see $|Z_{n+1}-Z_{n}|\leq H$, 
We then apply Lemma \ref{lemma12} to $\{Z_{n}\}_{n\geq 1}$ with $n=T$, $nx=(2\sqrt{C}+4H)\gamma$ and $ny=C$, and obtain that
	$$\mathbb{P}(Z_{T}\geq 2\sqrt{C}\gamma+4H\gamma,|V|\leq C)\leq \delta$$ 
At last, because $|W-Z_{T}|=|\sum_{k}-h_{k,s_{1}}+h_{k,s_{l_{k}+1}}|\leq KH$, we conclude that,
$$\mathbb{P}(|V|\leq C, |W|\geq KH+(4H +2\sqrt{C})\gamma)\leq 2\delta.$$  
\end{proof}
Note that $\textcircled{1}_{k}=v_{k}^{T}(P_{k}-I)^{T}h_{k}=\sum_{i=1}^{n}(P_{s_{i},a_{i}}^{T}h_{k}-h_{k,s_{i}})=\sum_{i=1}^{l_{k}}(P_{s_{i},a_{i}}^{T}h_{k}-h_{k,s_{i+1}})-h_{k,s_{1}}+h_{k,s_{l_{k}+1}}$. Let $X_{n}=\sum_{i=1}^{n}(P_{s_{i},a_{i}}^{T}h_{k}-h_{k,s_{i+1}})$.
Now it suffices to show that $\sum_{k}\sum_{s,a}v_{k,s,a}V(P_{s,a},h_{k})=O(TH)$ w.h.p.. Let $x^{2}$ denote the vector $[x_{1}^{2},...,x_{S}^{2}]^{T}$ for $x=[x_{1},...,x_{S}]^{T}$. Note that
\begin{equation}
    \begin{aligned}
    \sum_{k}\sum_{s,a}v_{k,s,a}V(P_{s,a},h_{k})&= \sum_{k}\sum_{s,a}v_{k,s,a}(P_{s,a}^{T}h_{k}^{2}-((P'_{k,s,a})^{T}h_{k})^{2})\\&+\sum_{k}\sum_{s,a}v_{k,s,a}(P'_{k,s,a}-P_{s,a})^{T}h_{k}(P'_{k,s,a}+P_{s,a})^{T}h_{k}. \label{C.5.0}
    \end{aligned}
\end{equation}
By the definition of $h_{k}$, we have that
$(P'_{k,s,a})^{T}h_{k}-h_{k,s}=\rho_{k}-r_{s,a}$. Then we obtain that,

\begin{equation}\label{C.5.1}
\begin{aligned}
    |\sum_{k,s,a}v_{k,s,a}(P_{s,a}^{T}h_{k}^{2}-((P'_{k,s,a})^{T}h_{k})^{2})|&=|\sum_{k,s,a}v_{k,s,a}(P_{s,a}^{T}h^{2}_{k})-h^{2}_{k,s}|+|\sum_{k,s,a}h^{2}_{k,s}-(h_{k,s}+\rho_{k}-r_{s,a})^{2}|
    \\& \leq  |\sum_{k,s,a}v_{k,s,a}(P_{s,a}^{T}h^{2}_{k})-h^{2}_{k,s}|+|\sum_{k,s,a}(\rho_{k}-r_{s,a})(2h_{k,s}+\rho_{k}-r_{s,a})|
    \\&\leq \sum_{k,s,a}v_{k,s,a}(P_{s,a}^{T}h_{k}^{2}-h_{k,s}^{2})+\sum_{k,s,a}v_{k,s,a}(2H+1) 
\end{aligned}
\end{equation}

According to Lemma (\ref{lemma10}), we have that, with probability $1-\delta$
\begin{equation}
    \sum_{k,s,a}v_{k,s,a}(P_{s,a}^{T}h_{k}^{2}-h_{k,s}^{2})\leq \sqrt{2T\gamma}H^{2}+KH^{2} \label{C.5.2}
\end{equation}

Combining  (\ref{C.5.1}) and  (\ref{C.5.2}), we have that, with probability $1-\delta$, it holds that
\begin{equation}
|\sum_{k,s,a}v_{k,s,a}(P_{s,a}^{T}h_{k}^{2}-((P'_{k,s,a})^{T}h_{k})^{2})|\leq \sqrt{2T\gamma}H^{2}+KH^{2} +T(2H+1)
\label{C.5.3}
\end{equation}

Assuming the good event $G$ occurs, the second term in  (\ref{C.5.0}) can be bounded by $4H^{2}\sum_{k,s,a}v_{k,s,a}\sqrt{\frac{S\gamma}{N_{k,s,a}}}$.  Combining this with  (\ref{C.5.3}),  we obtain that, with probability $1-\delta$, it holds that
\begin{equation}
    \sum_{k}\sum_{s,a}v_{k,s,a}V(P_{s,a},h_{k})\leq \sqrt{2T\gamma}H^{2}+KH^{2}+T(2H+1))+4\sqrt{2}H^{2}S\sqrt{AT\gamma} \label{C.5.4}
\end{equation}
The dominant term is the right hand side of  (\ref{C.5.4}) is $2TH$ when $T$ is large enough. Specifically, when $T\geq S^{2}AH^{2}\gamma$, we have $\sum_{k}\sum_{s,a}v_{k,s,a}V(P_{s,a},h_{k})\leq 12TH$. 

Let $C=12TH$ in Lemma \ref{L.5.1}, then it follows that
\begin{equation*}
\begin{aligned}
&\mathbb{P}(|\sum_{k}\textcircled{1}_{k}|\geq KH+  (4H+2\sqrt{12TH})\gamma \leq  \mathbb{P}(\sum_{k}\sum_{s,a}v_{k,s,a}V(P_{s,a},h_{k})\geq  12TH)+\\&\quad \quad\mathbb{P}(\sum_{k}\sum_{s,a}v_{k,s,a}V(P_{s,a},h_{k})\leq  12TH,|\sum_{k}\textcircled{1}_{k}|\geq KH+  (4H+2\sqrt{12TH})\gamma)\\&\quad \quad 
\leq 3\delta.
\end{aligned}
\end{equation*}

\subsection{ Proof of Lemma 4}\label{C.4}
\begin{lemma}\label{L.7.1} When $T\geq H^{2}S^{2}A\gamma$, with probability $1-\delta$, it holds that $\sum_{s,a}N^{(T)}_{s,a}V(P_{s,a},h^{*})\leq 49TH$
\end{lemma}
\begin{proof} Noting that $P_{s,a}^{T}h^{*}=h^{*}_{s}+\rho^{*}-r_{s,a}-reg_{s,a}$, we have
\begin{equation}
    \begin{aligned}
    \sum_{s,a}N_{s,a}^{(T)}V(P_{s,a},h^{*})&=\sum_{s,a}N_{s,a}^{(T)}(P_{s,a}^{T}h^{*2}-(P_{s,a}^{T}h^{*})^{2})\\
    &=\sum_{s,a}N_{s,a}^{(T)}  (P_{s,a}^{T}h^{*2}-h^{*2}_{s})+\sum_{s,a}N_{s,a}^{(T)}(reg_{s,a}+r_{s,a}-\rho^{*})(P_{s,a}^{T}h^{*}+h^{*}_{s})\\
    &\leq \sqrt{2T\gamma}H^{2}+KH^{2}+2H\sum_{s,a}N_{s,a}^{(T)}reg_{s,a}+2TH \label{C.7.1}
    \end{aligned} 
\end{equation}
with probability $1-\delta$. By definition of $B^{C}_{3,K+1}$, we have $\sum_{s,a}N_{s,a}^{(T)}reg_{s,a}\leq 22HS\sqrt{AT\gamma}$. By combining this inequality with  (\ref{C.7.1}), when $T\geq H^{2}S^{2}A\gamma$, we have 
$$\sum_{s,a}N^{(T)}_{s,a}V(P_{s,a},h^{*})\leq 2TH+H^{2}(44S\sqrt{AT\gamma}+\sqrt{2T\gamma}+K)\leq 49TH$$
holds with probability $1-\delta$.
\end{proof}

Assuming (\ref{C.7.1}) holds, we have that
\begin{equation}
\begin{aligned}
\sum_{k,s,a}v_{k,s,a}\sqrt{\frac{V(P_{s,a},h^{*})\gamma}{N_{k,s,a}}}&=\sum_{s,a}\sqrt{V(P_{s,a},h^{*})\gamma }  
	\sum_{k}v_{k,s,a}\sqrt{\frac{1}{N_{k,s,a}}}\\& \leq 2\sqrt{2}\sum_{s,a}\sqrt{N^{(T)}_{s,a}V(P_{s,a},h^*)\gamma }\\&
	\leq  2\sqrt{2SA\gamma}\sqrt{\sum_{s,a}N^{(T)}_{s,a}V(P_{s,a},h^*)}
	\\& \leq  21\sqrt{SAHT\gamma}.
\end{aligned}
\end{equation}
Here the first inequality is by Lemma \ref{lemma14} with $\alpha = \frac{1}{2}$, the second inequality is Jenson's inequality and (\ref{C.7.1}) implies the last inequality. Obviously, Lemma 4 follows by Lemma \ref{L.7.1}.

\subsection{ Proof of Lemma 5}\label{C.5}
Note that if we replace the reward $r_{s,a}$ by $r_{s,a}+reg_{s,a}$, then the MDP $M$ would be a \emph{flat} MDP. 
 According to Lemma 1, we have that, with probability $1-S^{2}T\delta$, for any $t\leq T$ and two different states $s,s'$, it holds that $$|\sum_{k=1}^{c(s,s',\mathcal{L}_{t_{k}})}\sum_{ts_{k}\leq i\leq te_{k}(\mathcal{L})-1}(r_{i}+reg_{s_{i},a_{i}}-\rho^{*})-c(s,s',\mathcal{L}_{t_{k}})\delta_{s,s'}^{*}|\leq (\sqrt{2 T\gamma}+1)H$$
At the same time, $B_{4,k}^{C}$ implies  (\ref{C.4.4}) is true for $t=t_{k}$. Then we have 
 $$|\sum_{k=1}^{c(s,s',\mathcal{L}_{t_{k}})}\sum_{ts_{k}\leq i\leq te_{k}(\mathcal{L})-1}(r_{i}-\hat{\rho}_{k})-c(s,s',\mathcal{L}_{t_{k}})\delta_{k,s,s'}|\leq (\sqrt{2T\gamma}+1)H+48HS\sqrt{AT\gamma}$$
Because $B_{3,k}^{C}$ occurs, $(t_{k}-1)|\rho^{*}-\hat{\rho}_{k}|\leq 26HS\sqrt{AT\gamma}$ and $\sum_{1\leq k'<k}reg_{s_{k'},a_{k'}}\leq 22HS\sqrt{AT\gamma}$. Let $N_{k,s,a,s'}=\sum_{1\leq t\leq t_{k}-1}I[s_{t}=s,a_{t}=a,s_{t+1}=s']$. Because $|a-b|\leq |a+c|+|b+d|+|c|+|d|$, by letting  
$$a= \sum_{k=1}^{c(s,s',\mathcal{L}_{t_{k}})}\sum_{ts_{k}\leq i\leq te_{k}(\mathcal{L})-1}(r_{i}-\rho^{*})-c(s,s',\mathcal{L}_{t_{k}})\delta^{*}_{s,s'},$$
 $$b=\sum_{k=1}^{c(s,s',\mathcal{L}_{t_{k}})}\sum_{ts_{k}\leq i\leq te_{k}(\mathcal{L})-1}(r_{i}-\rho^{*})-c(s,s',\mathcal{L}_{t_{k}})\delta_{k,s,s'},$$
  $$c=\sum_{k=1}^{c(s,s',\mathcal{L}_{t_{k}})}\sum_{ts_{k}\leq i\leq te_{k}(\mathcal{L})-1}reg_{s_{i},a_{i}}, \quad d=\sum_{k=1}^{c(s,s',\mathcal{L}_{t_{k}})}\sum_{ts_{k}\leq i\leq te_{k}(\mathcal{L})-1}(\rho^{*}-\hat{\rho}_{k}),$$ 
we have that 
 $$|N_{k,s,a,s'}(\delta_{k,s,s'}-\delta^{*}_{s,s'})|\leq |c(s,s',\mathcal{L}_{t_{k}})(\delta_{k,s,s'}-\delta^{*}_{s,s'})|\leq 2(\sqrt{2T\gamma}+1)H+96HS\sqrt{AT\gamma}$$
 and 
 \begin{equation}
     \begin{aligned}
     &\sum_{k}\sum_{s,a}v_{k,s,a}\sum_{s'}\sqrt{\frac{\hat{P}^{(k)}_{s,a,s'}|(\delta_{k,s,s'}-\delta_{s,s'}^{*})|}{N_{k,s,a}}}\\
     &= \sum_{k,s,a}\frac{v_{k,s,a}}{N_{k,s,a}}\sum_{s'}\sqrt{N_{k,s,a,s'}|(\delta_{k,s,s'}-\delta_{s,s'}^{*})|}\\
     &\leq KS^{2} \sqrt{ 2(\sqrt{2T\gamma}+1)H+96HS\sqrt{AT\gamma}   }\\
     &\leq 11KS^{\frac{5}{2}}A^{\frac{1}{4}}H^{\frac{1}{2}}T^{\frac{1}{4}}\gamma^{\frac{1}{4}}, \label{C.8.1}
     \end{aligned}
 \end{equation}
 where the first inequality holds because $\sum_{k,s,a}\frac{v_{k,s,a}}{N_{k,s,a}}\leq \sum_{k,s,a}\mathbb{I}[\pi_{k}(s)=a]\leq KS$.

\subsection{ Detailed Proof of Theorem 1}\label{C.6}
According to Lemma 2, the probability of bad event is bounded by $(6AT+12S^{2}A)SA\log(T)$ when $T\geq A\log(T)$ and $SA\geq 4$. 
We then consider to bound the regret when the good event occurs. We present more rigorous analysis compared to the proof sketch in Section 5.2. Recall that
\begin{equation*}
\begin{aligned}
\mathcal{R}_{k}&=v_{k}^{T}(\rho^{*}\textbf{1}-r_{k})\leq v_{k}^{T}(\rho_{k}\textbf{1}-r_{k})=v_{k}^{T}(P'_{k}-I)^{T}h_{k} \\
&=\underbrace{v_{k}^{T}(P_{k}-I)^{T}h_{k}}_{\textcircled{1}_{k}}+\underbrace{v_{k}^{T}(\hat{P}_{k}-P_{k})^{T}h^{*}}_{\textcircled{2}_{k}}+ \underbrace{v_{k}^{T}(P'_{k}-\hat{P}_{k})^{T}h_{k}}_{\textcircled{3}_{k}}+\underbrace{v_{k}^{T}(\hat{P}_{k}-P_{k})^{T}(h_{k}-h^{*})}_{\textcircled{4}_{k}};
\end{aligned}
\end{equation*}

\begin{equation}\label{rveqbd2}
\textcircled{2}_{k}\leq \sum_{s,a}v_{k,s,a} \bigg (2\sqrt{\frac{V(P_{s,a},h^{*})\gamma}{N_{k,s,a}}}+2\frac{H\gamma}{N_{k,s,a}} \bigg),
\end{equation}

\begin{equation}\label{rveqbd5}
\begin{aligned}
\sqrt{V(\hat{P}_{s,a}^{(k)},h_{k})}-\sqrt{V(P_{s,a},h^*)} &
\leq \sum_{s'}\sqrt{4H\hat{P}^{(k)}_{s,a,s'}|\delta_{k,s,s'}-\delta^*_{s,s'}| }
+\sqrt{4H^{2}\sqrt{ \frac{14S\gamma}{    N_{k,s,a}  } }}.
\end{aligned}
\end{equation}

Plugging (\ref{rveqbd5}) into (\ref{eqbd3}), we get that
\begin{equation}\label{rveqbd3}
\begin{aligned}
\textcircled{3}_{k} &\leq \sum_{s,a}v_{k,s,a}L_{2}(N_{k,s,a},\hat{P}^{(k)}_{s,a},h_{k})
 =\sum_{s,a}v_{k,s,a}\bigg (  2\sqrt{\frac{V(\hat{P}^{(k)}_{s,a},h_{k})\gamma}{N_{k,s,a}}}+12\frac{H\gamma}{N_{k,s,a}}+10\frac{H\gamma^{3/4}}{N_{k,s,a}^{3/4}} \bigg )
 \\&   \leq \sum_{s,a}v_{k,s,a}\bigg (  2\sqrt{\frac{V(P_{s,a},h^{*})\gamma}{N_{k,s,a}}} +4\sum_{s'}\sqrt{\frac{H\hat{P}^{(k)}_{s,a,s'}|\delta_{k,s,s'}-\delta^*_{s,s'}|  \gamma}{ N_{k,s,a} }}     +\frac{8HS^{\frac{1}{4}}\gamma^{3/4}}{N_{k,s,a}^{3/4}} +12\frac{H\gamma}{N_{k,s,a}}+10\frac{H\gamma^{3/4}}{N_{k,s,a}^{3/4}} \bigg ).
\end{aligned}
\end{equation}

Based on (\ref{eqbd4}), $B_{2,k}^{C}$ and the fact $|\delta_{k,s,s'}-\delta^{*}_{s,s'}|\leq 2H$, we have that
\begin{equation}\label{rveqbd4}
\begin{aligned}
\textcircled{4}_{k} &=\sum_{s,a}v_{k,s,a}(\hat{P}^{(k)}_{s,a}-P_{s,a})^{T}(h_{k}-h_{k,s}\textbf{1}-h^{*}+h^*_{s}\textbf{1})=\sum_{s,a}v_{k,s,a}\sum_{s'}(\hat{P}^{(k)}_{s,a,s'}-P_{s,a,s})(\delta^*_{s,s'}-\delta_{k,s,s'}) \\& \leq \sum_{s,a}v_{k,s,a}\sum_{s'}(2\sqrt{\frac{\hat{P}^{(k)}_{s,a,s'}\gamma}{N_{k,s,a}}}+\frac{3\gamma}{N_{k,s,a}}+\frac{4\gamma^{3/4}}{N_{k,s,a}^{3/4}})|\delta_{k,s,s'}-\delta^{*}_{s,s'}|\\&
\leq  2\sum_{k,s,a}v_{k,s,a}\bigg( \sum_{s'}\sqrt  {\frac{2H\hat{P}^{(k)}_{s,a,s'}|\delta_{k,s,s'}-\delta^*_{s,s'}| }{N_{k,s,a}} } +\frac{6SH\gamma}{N_{k,s,a}}+\frac{8SH \gamma^{3/4}}{N_{k,s,a}^{3/4}}     \bigg)
\end{aligned}
\end{equation}

Taking sum of RHS of (\ref{rveqbd2}), (\ref{rveqbd3}) and (\ref{rveqbd4}), based on the fact $S\geq 1$ we obtain that
\begin{equation}\label{eqmain2}
\begin{aligned}
\textcircled{2}_{k}+\textcircled{3}_{k}+\textcircled{4}_{k}&\leq \sum_{s,a}v_{k,s,a}\bigg(4\sqrt{\frac{V(P_{s,a},h^{*})\gamma}{N_{k,s,a}}}+20\frac{SH\gamma}{N_{k,s,a}}+7\sum_{s'}\sqrt{\frac{  H\hat{P}^{(k)}_{s,a,s'}|\delta_{k,s,s'}-\delta^*_{s,s'}|  \gamma} {N_{k,s,a}  }} +26\frac{SH\gamma^{3/4}}{N_{k,s,a}^{3/4}}\bigg)
\end{aligned}
\end{equation}
According to (\ref{eq5.2.1}),(\ref{eqmain2}) Lemma \ref{lemmabd2}, Lemma \ref{lemma5} and Lemma \ref{lemma14}, we obtain that when $T\geq S^{3}AH^{2}\gamma$ and $SA\geq 4$, with probability at least $1-20S^{3}A^{2}T\log(T)\delta$, it holds that
\begin{equation}\label{final}
\begin{aligned}
\mathcal{R}(T)=\sum_{k}\mathcal{R}_{k}& \leq KH+(4H+2\sqrt{TH})\gamma \\& +\sum_{k,s,a}v_{k,s,a}\bigg(4\sqrt{\frac{V(P_{s,a},h^{*})\gamma}{N_{k,s,a}}}+20\frac{SH\gamma}{N_{k,s,a}}+7\sum_{s'}\sqrt{\frac{  H\hat{P}^{(k)}_{s,a,s'}|\delta_{k,s,s'}-\delta^*_{s,s'}|  \gamma} {N_{k,s,a}  }} +26\frac{SH\gamma^{3/4}}{N_{k,s,a}^{3/4}}\bigg)  \\&
\leq KH+(4H+2\sqrt{TH})\gamma + 84\sqrt{SAHT\gamma}+77KS^{\frac{5}{2}}A^{\frac{1}{4}}HT^{\frac{1}{4}}\gamma^{\frac{3}{4}}\\&+20SH\gamma(1+2SA\log(T))+208S^{\frac{7}{4}}A^{\frac{3}{4}}T^{\frac{1}{4}}H\gamma^{\frac{3}{4}}=\tilde{O}(\sqrt{SATH}).
\end{aligned}
\end{equation}
Let $\delta_{1}=20S^{3}A^{2}T\log(T)\delta$. When $T\geq \{S^{12}A^{3}H^{2},H^{2}SA\kappa,HSA\log(T)^{2}\kappa,H^{2}S^{2}\log(T)\kappa\}$ where $\kappa =\log(\frac{40S^{3}A^{2}T\log(T)}{\delta_{1}})$, with probability $1-\delta_{1}$, we have that
\begin{equation*}
\mathcal{R}(T)\leq 490\sqrt{SATHlog(\frac{40S^{2}A^{2}Tlog(T)}{\delta_{1}})}.
\end{equation*}
\textbf{The selection of $p_{1}$:} Let $p_{1}(S,A,H,\log(\frac{1}{\delta}))=64\log(\frac{1}{\delta}))^{2}(S^{4}A^{4}H^{6}+S^{4}A^{4}H^{4}+S^{6}A^{2}H^{6})+\\S^{12}A^{3}H^{3}+100$. When $T\geq p_{1}(S,A,H,\log(\frac{1}{\delta}))$ and $S,A\geq 20$, we have that $T\geq S^{12}A^{3}H^{3}$ and $\frac{T}{\log^{3}(T)}\geq \sqrt{T}\geq 8\log(\frac{1}{\delta})\max\{ S^{2}A^{2}H^{3},S^{3}AH^{3}   \}\geq \frac{1}{\log(T)}\max \{H^{2}SA\kappa,HSA\log(T)^{2}\kappa,H^{2}S^{2}\log(T)\kappa\}$, since $8SA\geq  \frac{\kappa}{\log(\frac{1}{\delta})\log(T)}$. Therefore, $T\geq \max\{S^{12}A^{3}H^{2},H^{2}SA\kappa,HSA\log(T)^{2}\kappa,H^{2}S^{2}\log(T)\kappa\}$.


\section{ Proof of Corollary 1}\label{D}

\begin{algorithm}[tb]
   \caption{LD: Learn the Diameter}
   \label{alg:example3}
\begin{algorithmic}\label{algr3}
   \STATE {\bfseries Input: $T_{0}$, $\delta_{0}$, $x\neq y\in \mathcal{S}$}
   \STATE{$t\leftarrow 1$, $I_{x,y}(t)\leftarrow 0$, $t_{lu}^{(1)}\leftarrow 1$, $t_{lu}^{(2)}\leftarrow 1$, $\pi^{(1)}(s),\pi^{(2)}(s) \leftarrow \mbox{arbitrary policy}$,  $\forall s$;}
   \STATE{$N_{s,a}^{(1)}(t)\leftarrow 0$,$N_{s,a}^{(2)}(t)\leftarrow 0$, $N_{s,a,s'}^{(1)}(t)\leftarrow 0$, $N_{s,a,s'}^{(2)}(t)\leftarrow 0$ $\hat{P}^{(1)}_{s,a,s'}(t)\leftarrow 0$,$\hat{P}^{(2)}_{s,a,s'}(t)\leftarrow 0$, $\forall s,a,s'$;}
   \IF{current state is not $x$}
   \STATE{$r^{(t)}\leftarrow \textbf{1}_{x}$;}
   \ELSE
   \STATE{$r^{(t)}\leftarrow \textbf{1}_{y}$;}
   \ENDIF
   \FOR{$t=1,2,...T_{0}$}

   \IF{$r^{(t)}=\textbf{1}_{x}$}
    \STATE{$I_{x,y}(t)\leftarrow 0$;}
    \IF{$\exists (s,a)$, s.t. $N^{(1)}_{s,a}(t)\geq 2N^{(1)}_{s,a}(t^{(1)}_{lu})$ or $t=1$}
   \STATE{$t_{lu}^{(1)}\leftarrow t$;}\STATE{ update $\mathcal{P}$ as: $\mathcal{P}=\{P'|\forall (s,a)$,$|P'_{s,a}-\hat{P}^{(1)}_{s,a}(t)|_{1}\leq \sqrt{\frac{14SA\log(2AT_{0}/\delta_{0})}{\max\{N^{(1)}_{s,a}(t),1\}}} $}
   \STATE{$P_{1}\leftarrow \mathop{\arg\max}\limits_{Q\in \mathcal{P}}\rho(mdp(Q^{(x,y)},\textbf{1}_{x}))$;}
   \STATE{$\pi^{(1)}\leftarrow $ optimal policy for $mdp(P_{1}^{(x,y)},\textbf{1}_{x})$;}
    \ENDIF
   \STATE{Execute $\pi^{(1)}(s_{t})$, get $r_{t}=r^{(t)}(s_{t},a_{t})$ and transits to $s_{t+1}$;}
   \IF{$s_{t+1}=x$}
   \STATE{$r^{(t+1)}=\textbf{1}_{y}$}
   \ENDIF
   \ELSE
     \STATE{$I_{x,y}(t)\leftarrow 1$;}
    \IF{$\exists (s,a)$, s.t. $N^{(2)}_{s,a}(t)\geq 2N^{(2)}_{s,a}(t^{(2)}_{lu})$ or $t=0$}
   \STATE{$t^{(2)}_{lu}\leftarrow t$;}\STATE{ update $\mathcal{P}$ as: $\mathcal{P}=\{P'|\forall (s,a)$,$|P'_{s,a}-\hat{P}^{(2)}_{s,a}(t)|_{1}\leq \sqrt{\frac{14SA\log(2AT_{0}/\delta_{0})}{\max\{N^{(2)}_{s,a}(t),1\}}} $}
  \STATE{$P_{2}\leftarrow \mathop{\arg\max}\limits_{Q\in \mathcal{P}}\rho(mdp(Q^{(y,x)},\textbf{1}_{y}))$;}
  \STATE{$\pi^{(2)}\leftarrow $ optimal policy for $M_{2}'$;}
    \ENDIF
   \STATE{Execute $\pi^{(2)}(s_{t})$, get $r_{t}=r^{(t)}(s_{t},a_{t})$ and transits to $s_{t+1}$;}
   \IF{$s_{t+1}=y$}
   \STATE{$r^{(t+1)}=\textbf{1}_{x}$}
   \ENDIF
   \ENDIF
   \STATE{Update:}
   \STATE{ $N^{(1)}_{s,a}(t+1)=\sum_{i=1}^{t}I[s_{t}=s,a_{t}=a,r^{(t)}=\textbf{1}_{x}]$;$N_{s,a}^{(2)}(t)=\sum_{i=1}^{t}I[s_{t}=s,a_{t}=a,r^{(t)}=\textbf{1}_{y}]$}
   \STATE{ $N_{s,a,s'}^{(1)}(t+1)=\sum_{i=1}^{t}I[s_{t}=s, a_{t}=a, s_{t+1}=s', r^{(t)}=\textbf{1}_{x}]$;$N_{s,a,s'}^{(2)}(t+1)=\sum_{i=1}^{t}I[s_{t}=s, a_{t}=a, s_{t+1}=s', r^{(t)}=\textbf{1}_{y}]$;}
   \STATE{ $\hat{P}^{(1)}_{s,a,s'}(t+1)=\frac{N_{s,a,s'}^{(1)}(t+1)}{\max\{N_{s,a}^{(1)}(t+1),1\}}$;$\hat{P}^{(2)}_{s,a,s'}(t+1)=\frac{N_{s,a,s'}^{(2)}(t+1)}{\max\{N_{s,a}^{(2)}(t+1),1\}}$.}
   \ENDFOR
   \STATE {\bfseries Return:$(\frac{|\{t|r_{t}=\textbf{1}_{y}\}|}{|\{t|s_{t}=y,r^{(t-1)}=\textbf{1}_{y}\}|},\frac{|\{t|r_{t}=\textbf{1}_{x}\}|}{|\{t|s_{t}=x,r^{(t-1)}=\textbf{1}_{x}\}|})$. }
\end{algorithmic}
\end{algorithm}


In this section we consider to learn MDPs with finite diameter. According to Theorem 1, in order to reach an $\tilde{O}(\sqrt{DSAT})$ upper bound for the regret, it suffices to provide a real number $H$ such that $sp(h^{*})\leq H\leq D$ within $o(\sqrt{T})$ steps. For a transition model $P$, we use $P^{(x,y)}$ to denote the transition model satisfying that $P^{(x,y)}_{s,a}=P_{s,a}$ when $s\neq x$, and $P^{(x,y)}_{s,a}=\textbf{1}_{y}$\footnote{We use $\textbf{1}_{y}$ to denote the vector $v$ satisfying $v_{s}=I[s=y],\forall s$.} when $s=x$, $\forall a$. Let $D_{xy}=\min\limits_{\pi:\mathcal{S}\to \Delta_{\mathcal{A}}}T^{\pi}_{x\to y}$, then we try to learn $D_{xy}$ directly. 

In  Algorithm \ref{algr3}, when we start from $x$, we target to reach $y$ as soon as possible by employing a UCRL2-like algorithm. Once we reach $y$, we change the target to achieve $x$. Let $mdp(P,r)$ denote the MDP with transition model $P$ and reward function $r$.
 We maintain the two learning process separately, so they are corresponding to running two independent learning processes, which learn  $mdp(P^{(y,x)},\textbf{1}_{y})$ and $mdp(P^{(x,y)},\textbf{1}_{x})$  respectively. Based on  Algorithm \ref{algr3}, we can get a close approximation for $D_{xy}$ within $T^{\frac{1}{4}}$ steps. Without loss of generality, we assume $T^{\frac{1}{4}}$ is an integer.
\begin{lemma} \label{Diameter}
	When $T\geq (136D^{3}S\sqrt{A\gamma})^{8}$, for any $x\neq y\in \mathcal{S}$, let $(\hat{D}_{xy},\hat{D}_{yx})$ be the output of Algorithm \ref{algr3} with ($T^{1/4},\delta,x,y$) as the input, then with probability $1-8SAT^{\frac{1}{2}}\delta$, it holds that $|\hat{D}_{xy}-D_{xy}|\leq 1$ and $|\hat{D}_{yx}-D_{yx}|\leq 1$.
\end{lemma}
\begin{proof}[\textbf{Proof of Corollary 1}] Obviously, an MDP with finite diameter is weak-communicating. We run Algorithm \ref{algr3} for all $s\neq s'$ with  $T_{0}=T^{1/4}$  and $\delta_{0}=\delta$ (without loss of generality, we assume that $T^{\frac{1}{4}}$ is an integer.). Denote the output of Algorithm \ref{algr3} with input ($T^{1/4},\delta,s,s'$) as $(\hat{D}_{ss'},\hat{D}_{s's})$. Let $\hat{H}=\max\limits_{s,s'}\hat{D}_{ss'}+1$. According to Lemma \ref{Diameter}, $sp(h^{*})\leq \max\limits_{s,s'}D_{ss'}\leq \hat{H}\leq D+2$ with probability $1-8S^{3}AT^{\frac{1}{2}}\delta$. We then execute Algorithm 1 with $H=\hat{H}$ for $T-S(S-1)T^{\frac{1}{4}}$ steps. Since the total number of time steps for performing Algorithm \ref{algr3} is at most $S^{2}T^{\frac{1}{4}}$, the regret in the first stage is at most $S^{2}T^{\frac{1}{4}}$. According to Theorem 1,  when
	$T\geq2\max\{(136D^{3}S\sqrt{A\kappa})^{8},S^{12}A^{3}D^{2},DSAlog^{2}(T)\kappa,D^{2}SA\kappa, D^{2}S^{2}log(T)\kappa\}$ where $\kappa =log(\frac{44S^{2}A^{2}Tlog(T)}{\delta_{1}})$,  the regret can be bounded as
	$$\mathcal{R}(T)\leq 491\sqrt{SATD(log(\frac{S^{3}A^{2}Tlog(T)}{\delta})}.$$,with probability $1-\delta$, the regret is at most $491\sqrt{SATD\log(\frac{44S^{2}A^{2}T\log(T)}{\delta_{1}})}$ .
\end{proof}

\textbf{The selection of $p_{2}$:} Let $p_{2}(S,A,D,\log(\frac{1}{\delta}))=4(136D^{3}S\sqrt{A})^{16}(8SA)^{8}+log(\frac{1}{\delta})^{8}10^{16}$.  When $T\geq p_{2}(S,A,D,\log(\frac{1}{\delta}))$ and $S,A,D\geq 20$, $\frac{T}{\log(\frac{1}{\delta})^{4}\log(T)^{4}}\geq \sqrt{T} \geq  2(136D^{3}S\sqrt{A})^{8}(8SA)^{4}\geq\\ \frac{2(136D^{3}S\sqrt{A\kappa})^{8}}{\log(\frac{1}{\delta})^{4}\log(T)^{4}}$, since $8SA \geq \frac{\kappa}{\log(\frac{1}{\delta})\log(T)} $. Therefore, $T\geq \max\{2(136D^{3}S\sqrt{A\kappa})^{8},\\  2(D^{3}S\sqrt{A})^{16} \}= 2\max\{(136D^{3}S\sqrt{A\kappa})^{8},S^{12}A^{3}D^{2},DSAlog^{2}(T)\kappa,D^{2}SA\kappa, D^{2}S^{2}\log(T)\kappa\}$ .

\subsection{ Proof of Lemma \ref{Diameter}}

In Algorithm \ref{algr3}, we maintain two learning process. We use $I_{x,y}(t)$ to indicate whether the $t$-th step is contained by the first process. For $t\geq T_{0}+1$, we set $I_{x,y}(t)=0$. Let $M_{1}$ be the MDP with transition probability $P^{(x,y)}$ and reward $\textbf{1}_{y}$, and $h^{(1)}$, $\rho^{(1)}$ denote the optimal bias function and the optimal average reward of $M_{1}$ respectively. In the same way we define $M_{2}$, $h^{(2)}$ and $\rho^{(2)}$ according to transition probability $P^{(y,x)}$ and reward $\textbf{1}_{x}$.

For the first process, 
the regret $\mathcal{R}^{(1)}=\sum_{1\leq t\leq T_{0},I_{x,y}(t)=1}\rho^{(1)}+\sum_{1\leq t\leq T_{0},s_{t+1}=y, I_{x,y}(t)=1}(\rho^{(1)}-1)=(t^{(1)}+k^{(1)})\rho^{(1)}-k^{(1)}$, where $t^{(1)}=\sum_{1\leq t\leq T_{0}}I_{x,y}(t)$ and $k^{(1)}=|\{t\leq T_{0}|s_{t+1}=y, I_{x,y}(t)=1\}|$.  We aim to prove that with probability $1-p$ for some $p\in (0,1)$, it holds that
\begin{equation}
|\mathcal{R}_{1}|\leq 34DS\sqrt{AT_{0}\gamma}.
\label{C.6.1}
\end{equation}  
Because $\rho^{(1)}=\frac{1}{D_{xy}+1}$, assuming  (\ref{C.6.1}) holds, we have $|\frac{t^{(1)}}{k^{(1)}}-D_{xy}|\leq \frac{68D^{2}S\sqrt{AT_{0}\gamma}}{k^{(1)}}$. 
On the other side, we define $t^{(2)}=\sum_{1\leq t\leq T_{0}}(1-I_{x,y}(t)) $, $k^{(2)}=|\{t\leq T_{0}|s_{t+1}=x, I_{x,y}(t)=0\}|$, and thus $\mathcal{R}_{2}=(t^{(2)}+k^{(2)})\rho^{(2)}-k^{(2)}$.  Assuming 
\begin{equation}
   |\mathcal{R}_{2}|\leq 34DS\sqrt{AT_{0}\gamma} \label{C.10.2}
\end{equation}
holds, it follows that $|\frac{t^{(2)}}{k^{(2)}}-D_{yx}|\leq \frac{68D^{2}S\sqrt{AT_{0}\gamma}}{k^{(2)}}$. Noticing that $|k^{(1)}-k^{(2)}|\leq 1$ and $t^{(1)}+t^{(2)}=T_{0}$, we derive that $k^{(1)}\geq \frac{T_{0}}{2D} $ and $k^{(2)}\geq \frac{T_{0}}{2D}$. Therefore, we get that
$$|\frac{t^{(1)}}{k^{(1)}}-D_{xy}|\leq \frac{68D^{2}S\sqrt{AT_{0}}}{k^{(1)}}\leq \frac{136D^{3}S\sqrt{A\gamma}}{\sqrt{T_{0}}}$$
$$|\frac{t^{(2)}}{k^{(2)}}-D_{yx}|\leq \frac{68D^{2}S\sqrt{AT_{0}}}{k^{(2)}}\leq \frac{136D^{3}S\sqrt{A\gamma}}{\sqrt{T_{0}}}.$$
Because $\sqrt{T_{0}}\geq 136D^{3}S\sqrt{A\gamma}$, we conclude that $|\frac{t^{(1)}}{k^{(1)}}-D_{xy}|\leq 1$ and $|\frac{t^{(2)}}{k^{(2)}}-D_{yx}|\leq 1$ with probability $1-2p$.

Theorem2 in [Jaksch et al., 2010] provides a solid foundation to prove  (\ref{C.6.1}) holds with high probability. Following the analysis of this theorem, we have some lemmas below.
\begin{lemma}\label{9pre}
	Let $X_{1},X_{2},...$ be i.i.d. discrete random variables with support $\mathcal{X}$.  Let $I_{n}\in \{0,1\}$ be random variables in $\{0,1\}$ for $n=1,2,...$.  Assume that for each $n$, $X_{n}$ is independent of $\{I_{1},...,I_{n}\}$.  
	Let $a_{k}=\min\{i\geq 1| \sum_{j=1}^{i}I_{j}\geq k  \}$. For any $k\geq 1$, if $a_{k}< \infty$ with probability 1, then the joint distribution of $(X_{a_{1}},...,X_{a_{k}})$ is the same as the joint distribution of $(X_{1},...,X_{k})$, which means $X_{a_{1}},...,X_{a_{k}}$ are i.i.d. random variables. 
\end{lemma}
\begin{proof} When $k=1$, for each $i\geq 1$, conditioning on $a_{1}=i$, the distribution of $X_{a_{k}}$ is the same as the distribution of $X_{1}$, since $X_{i}$ is independent of $(X_{1},...,X_{i-1},I_{1},...,I_{i})$. Because $a_{k}< \infty$ with probability 1, then we have $\mathbb{P}(X_{a_{k}}=x)=\sum_{i=1}^{\infty}\mathbb{P}(a_{k}=i)\mathbb{P}(X_{1}=x)=\mathbb{P}(X_{1}=x)$ for any $x\in \mathcal{X}$.  For $n\geq 2$, we assume that this lemma holds for $k=n-1$.
In the same way we have that for any $x\in \mathcal{X}$, $\mathbb{P}(X_{a_{n}}=x|a_{1},a_{2},...,a_{n},X_{1},...,X_{a_{n}-1})=\mathbb{P}(X_{1}=x)$. It then follows that for any $(x_{1},...,x_{n})\in \mathcal{X}^{n}$,  $\mathbb{P}(X_{a_{1}}=x_{1},...,X_{a_{n}}=x_{n})=\mathbb{P}(X_{a_{1}}=x_{1},...,X_{a_{n-1}}=x_{n-1})\mathbb{P}(X_{a_{n}}=x_{n}|X_{a_{1}}=x_{1},...,X_{a_{n-1}}=x_{n-1})=\mathbb{P}(X_{a_{1}}=x_{1},...,X_{a_{n-1}}=x_{n-1})\mathbb{P}(X_{1}=x_{n})=\Pi_{i=1}^{n}\mathbb{P}(X_{1}=x_{i})$. Then the conclusion follows by induction.
\end{proof}


\begin{lemma}
	With probability $1-\frac{\delta}{60T_{0}^{6}}$, in any episode, the true transition probability $P$ is in $\mathcal{P}$.
\end{lemma}
\begin{proof}
	Because the rewards $\{r_{s,a} \}_{s\in \mathcal{S},a\in \mathcal{A}}$ are assumed to be known in the beginning, it suffices to make sure $|P_{s,a}-\hat{P}^{(1)}_{s,a}|_{1}\leq \sqrt{\frac{14SA\log(2AT_{0}/\delta_{0})}{\max\{N_{s,a}^{(1)}(t),1\}}}$. 
	
	To apply Lemma \ref{9pre}, we have to make sure $a_{k}\leq \infty$ with probability 1 for  $\forall k\leq T_{0}$. But it's easy to see that, if we let $I_{n}=I_{x,y}(t(n,s,a))$ for $n\leq T_{0}$ where $t(n,s,a)$ is the first time $(s,a)$ is visited for $n$ times (if the visit number of $(s,a)$ is less than $n$, we set $t(n,s,a)=T_{0}+1$ and $I_{n}=I_{x,y}(T_{0}+1)=0$ ). 
	For $T_{0}+1\leq n \leq 2T_{0}$, we set $I_{n}=1$ , then it follows $a_{k}\leq 2T_{0}$ for $\forall k\leq T_{0}$. 
	Note that $I_{x,y}(t)$ is a function of the random events before the $t$-th round, and thus $I_{x,y}(t)$ is obviously independent of subsequent states $(s_{t+1},s_{t+2},...)$.  When $n\geq T_{0}+1$, $I_{n}$ is independent of all other random variables. 
	As a result, for any $k\leq T_{0}$, the conclusion of Lemma \ref{9pre} holds for $\hat{P}_{s,a,1},\hat{P}_{s,a,2},...$ and $I_{1},I_{2},...$, where $\hat{P}_{s,a,i}\in \mathbb{R}^{S}$ is the result of the $i$-th try of executing $a$ in $s$.
	
	Because 
	 $N^{(1)}_{s,a}(t)\leq T_{0}$, according to Lemma \ref{9pre}, the distribution of $\hat{P}^{(1)}_{s,a}(t)$ is the same as the distribution of $\frac{1}{N^{(1)}_{s,a}(t)}\sum_{i=1}^{N^{(1)}_{s,a}(t)}P_{s,a,i}$, where $P_{s,a,1},P_{s,a,2},...$ are i.i.d. distributed obeying multinomial distribution with parameter $P_{s,}$. Based on the analysis in Lemma 17 [Jaksch et al., 2010],
	 we conclude that with probability $1-\frac{\delta}{60T_{0}^{6}}$, , for any $t\leq T_{0}$ and any $(s,a)$, it holds that
	 $$ |P_{s,a}-\hat{P}_{s,a}^{(1)}(t)|\leq \sqrt{\frac{14SA\log(2AT_{0}/\delta_{0})}{\max\{N_{s,a}^{(1)}(t),1  \}}}$$
\end{proof}

\begin{lemma}
	Let $P_{k}'$ denote the transition model of the optimal extended MDP in the $k$-th episode, and $u_{k}$ denote the optimal bias function of $mdp(P_{k}',\textbf{1}_{y})$. Then we have $sp(u_{k})\leq D_{y}:=sup_{z\neq y}D_{zy}$.
\end{lemma}
\begin{proof}
	Firstly, it's easy to see that $u_{k,y}\geq u_{k,z}$ for any $z\in \mathcal{S}$. Assume that there exists $z$ such that $u_{k,y}-u_{k,z}>D_{y}\geq D_{zy}$. We can design a nonstationary policy to achieve better value for $u_{k,z}$: in the first, we start from $z$ following some policy to reach $y$ as quickly as possible. Because the true transition model $P\in \mathcal{P}$ in each episode, we can reach $y$ within $D_{zy}$ steps in expectation. After reaching $y$, we follow the original optimal policy. Let $R_{t}(s)$ be the optimal $t$-step accumulative reward starting from $s$ and $\rho$ be the corresponding optimal average reward. According to the definition of optimal bias function, we have $lim_{t\to \infty}R_{t}(z)-\rho t=u_{k,z}\geq lim_{t\to \infty}R_{t-D_{zy}}(y)-\rho t\geq u_{k,y}-D_{zy}$. Therefore, $sp(u_{k})\leq \max_{z}\{u_{k,y}-u_{k,z}\}\leq D_{zy}$.
\end{proof}

According to the derivation in Section 4 [Jaksch et al., 2010], we have that
\begin{equation}\label{C.9.1}
\begin{aligned}
\mathcal{R}(mdp(P^{(x,y)},\textbf{1}_{y}),T_{0})& \leq |\sum_{k}v_{k}^{T}(P'_{k}-I)^{T}u_{k}| \leq  |\sum_{k}v_{k}^{T}(P_{k}-I)^{T}u_{k}|+|\sum_{k}v_{k}^{T}(P'_{k}-P_{k})u_{k}|\\ & \leq
D\sqrt{\frac{5}{2}T\log(\frac{8T_{0}}{\delta_{0}})}+ DSA\log_{2}(\frac{8T}{SA})+ (2D\sqrt{14S\log(\frac{2AT_{0}}{\delta_{0}})} +2)(\sqrt{2}+1)\sqrt{T}  
\end{aligned}
\end{equation}
holds with probability $1-2T_{0}\frac{\delta}{12T_{0}^{5/4}}-\frac{\delta}{60T_{0}^{6}}$. 

\textbf{Remark:} We can prove  (\ref{C.9.1}) holds with high probability for all $t\leq T_{0}$ in the same way. As a result, we conclude that, with probability $1-3SAT_{0}^{2}\delta$, for any $t\leq T_{0}$, it holds that  $\mathcal{R}(mdp(P^{(x,y)},\textbf{1}_{y}),t)\leq 34DS\sqrt{AT_{0}\gamma}$.

With a slight abuse of notations, we use $reg_{s,a}$ to denote the single step regret for $mdp(P^{(x,y)},\textbf{1}_{y})$. Noting that $sp(h^{(1)})=\frac{D_{y}}{1+D_{xy}}\leq D$, according to (\ref{3.1}) in Lemma \ref{lemma13}, for any $t\leq T_{0}$ it holds that $$\mathcal{R}(mdp(P^{(x,y)},\textbf{1}_{y}),t)-\sum_{i=1}^{t}reg_{s_{i},a_{i}}\geq -2\sqrt{T_{0}\gamma}D-D\geq -34DS\sqrt{AT_{0}\gamma}$$ with probability $1-\delta$.
Therefore, we conclude that with probability $1-4SAT_{0}^{2}\delta$, it holds that $|\mathcal{R}(mdp(P^{(x,y)},\textbf{1}_{y}),t)|\leq 34DS\sqrt{AT_{0}\gamma}$ for any $t\leq T_{0}$.

\end{document}